\documentclass{article}
\usepackage{iclr2027_conference,times}

\usepackage{amsmath,amsfonts,bm}

\def\eqref#1{equation~\ref{#1}}
\def\1{\bm{1}}

\DeclareMathAlphabet{\mathsfit}{\encodingdefault}{\sfdefault}{m}{sl}
\SetMathAlphabet{\mathsfit}{bold}{\encodingdefault}{\sfdefault}{bx}{n}

\newcommand{\E}{\mathbb{E}}

\newcommand{\fluAceTailQuarterLambdaFour}{1.00}
\newcommand{\fluAceTailQuarterLambdaHalf}{1.06}
\newcommand{\fluAceTailQuarterLambdaOne}{1.09}
\newcommand{\fluAceTailQuarterLambdaQuarter}{1.07}
\newcommand{\fluAceTailQuarterLambdaTwo}{1.08}
\newcommand{\fluAltCutoffA}{1.00}
\newcommand{\fluAltCutoffB}{1.01}
\newcommand{\fluAltCutoffC}{1.03}
\newcommand{\fluAltDecayA}{1.02}
\newcommand{\fluAltDecayB}{1.00}
\newcommand{\fluAltDecayC}{1.02}
\newcommand{\fluAltDecayD}{1.02}

\newcommand{\fluAltEarlyA}{1.01}
\newcommand{\fluAltEarlyB}{1.02}
\newcommand{\fluAltEarlyC}{1.03}
\newcommand{\fluAltEarlyD}{1.05}
\newcommand{\fluAltEarlyE}{1.04}

\newcommand{\fluAltMagA}{1.00}
\newcommand{\fluAltMagB}{1.03}
\newcommand{\fluAltMagC}{1.03}
\newcommand{\fluAltMagD}{1.05}

\newcommand{\fluAltSCReset}{1.00}
\newcommand{\fluCrossChurnHighACE}{1.07}

\newcommand{\fluCrossChurnLowACE}{1.06}

\newcommand{\fluCrossDirectionACE}{1.05_{\scriptscriptstyle\pm 0.02}}

\newcommand{\fluCrossGuidanceHighACE}{1.09}

\newcommand{\fluCrossGuidanceLowACE}{1.04}

\newcommand{\fluCrossNoiseHighACE}{1.09}

\newcommand{\fluCrossNoiseLowACE}{1.03}

\newcommand{\fluCrossODEACE}{1.05}

\newcommand{\fluCrossReferenceACE}{1.08}

\newcommand{\fluCrossSizeLACE}{1.42}
\newcommand{\fluCrossSizeLFull}{1.43}
\newcommand{\fluCrossSizeMACE}{1.27}
\newcommand{\fluCrossSizeMFull}{1.20}
\newcommand{\fluCrossStepsLongACE}{1.02}

\newcommand{\fluCrossStepsShortACE}{1.03}

\newcommand{\fluHeldACE}{1.05}

\newcommand{\fluHeldPermuted}{1.02}
\newcommand{\fluHeldRandom}{1.07}
\newcommand{\fluMainACENativeL}{1.45_{\scriptscriptstyle\pm 0.03}}
\newcommand{\fluMainACENativeM}{1.26_{\scriptscriptstyle\pm 0.05}}
\newcommand{\fluMainACESharedB}{1.08_{\scriptscriptstyle\pm 0.01}}
\newcommand{\fluMainACESharedL}{1.42_{\scriptscriptstyle\pm 0.04}}
\newcommand{\fluMainACESharedM}{1.27_{\scriptscriptstyle\pm 0.04}}
\newcommand{\fluMainFullB}{1.05_{\scriptscriptstyle\pm 0.03}}
\newcommand{\fluMainFullL}{1.43_{\scriptscriptstyle\pm 0.04}}
\newcommand{\fluMainFullM}{1.20_{\scriptscriptstyle\pm 0.05}}

\usepackage{url}
\usepackage{booktabs}
\usepackage{array}
\usepackage{tabularx}
\usepackage{graphicx}
\usepackage{amsmath}
\usepackage{amssymb}
\usepackage{amsthm}
\newtheorem{assumption}{Assumption}
\newtheorem{definition}{Definition}
\newtheorem{proposition}{Proposition}

\usepackage{xcolor}
\usepackage{multirow}
\usepackage{bm}
\usepackage{algorithm}
\usepackage{algpseudocode}
\usepackage{enumitem}
\usepackage{hyperref}
\usepackage{mdframed}

\mdfdefinestyle{examplebox}{%
  backgroundcolor=black!2,
  linecolor=black!25,
  linewidth=0.5pt,
  innerleftmargin=5pt,
  innerrightmargin=5pt,
  innertopmargin=4pt,
  innerbottommargin=4pt,
  skipabove=3pt,
  skipbelow=5pt}

\newcommand{\metrichead}[2]{\shortstack[c]{\strut #1\\\strut #2}}

\title{A Dominant Self-Conditioning Direction\\
Drives Repetition in Unconditional\\
Continuous Diffusion Language Models}

\author{Shuai Zhang$^{1,2}$, \, Zijie Chen$^{2}$, \, Huachuan Qiu$^{2}$, \, Hongliang He$^{2}$, \, Lun Du$^{3,\dagger}$, \, Zhenzhong Lan$^{2,\dagger}$ \\[3pt]
  {\normalfont $^{1}$Zhejiang University \quad $^{2}$Westlake University \quad $^{3}$Ant Group} \\[3pt]
  {\normalfont \texttt{zhangshuai@westlake.edu.cn} \quad \texttt{lanzhenzhong@westlake.edu.cn}} \\[3pt]
  {\normalfont\small $^{\dagger}$Corresponding authors}\thanks{Code: \url{https://github.com/ZhangShuai1230/ACE-DLM}}}
  \iclrfinalcopy

\begin{document}
\maketitle
\lhead{Preprint}
\begin{abstract}
Continuous diffusion language models offer an alternative to autoregressive generation, but their
generations may suffer from repetition. We find that unconditional generations from ELF, a recent
family of continuous diffusion language models, are more repetitive than human text, while Gen-PPL, a common
likelihood-based metric, gives lower perplexity to repetitive generations and can conceal this problem while biasing quality evaluation. Our
analysis links this behavior to a self-conditioning feedback loop
in which clean-embedding predictions are repeatedly carried into subsequent denoising steps,
driving representations toward an effectively one-dimensional contractive attractor associated
with repetition. Based on this mechanism, we introduce Attractor-Contrast-Escape (ACE), a
training-free inference-time intervention that estimates a repetition direction by contrasting denoising
paths trapped in repetition with paths relatively free of repetition and subtracts it from the
self-conditioning feedback during denoising.
Using a direction estimated only once on ELF-B, ACE reduces
mean $4$-gram self-repetition rate from $7.28\%$ to $4.48\%$, while retaining
competitive results on several text-quality metrics beyond Gen-PPL. The direction remains effective across ELF sizes and inference
configurations, and ACE also generalizes to other unconditional self-conditioned continuous diffusion language
models. These results identify self-conditioning feedback as a source
of repetition in continuous diffusion language models and show that ACE can directly mitigate this repetition during inference.

\end{abstract}

\vspace{-0.10in}
\begin{figure}[!h]
  \centering
    \includegraphics[width=0.88\linewidth]{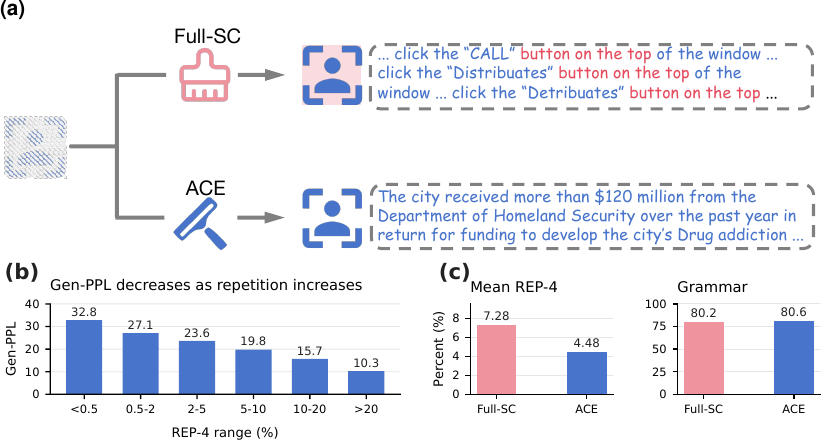}
  \caption{\textbf{Full-SC produces repetitive text and Gen-PPL favors it, whereas ACE reduces
  repetition with competitive text quality.} (a) An illustration of text generation from the same noisy
  state: Full-SC yields repetitive text, whereas ACE suppresses repetition during denoising. (b)
  Gen-PPL decreases across repetition bins. (c) ACE lowers mean REP-4 while preserving
  grammar.}
  \label{fig:basin}
\end{figure}
\vspace{-0.08in}

\section{Introduction}
\label{sec:intro}

Continuous diffusion language models (DLMs) generate text by denoising in a continuous embedding
space. Relative to discrete DLMs, this continuous state exposes a
differentiable denoising trajectory that provides a direct interface for gradient-based control and
classifier-free guidance~\citep{li2022diffusion,hu2026elf}. Self-conditioning improves generation
quality by enabling the denoiser to refine its clean-embedding prediction across denoising
steps~\citep{chen2022analog}. Related studies interpret self-conditioning as fixed-point iteration
for few-step distillation~\citep{yoo2026self} and address its training--inference
mismatch~\citep{li2026improving}.

However, recent continuous DLMs such as ELF~\citep{hu2026elf} can produce highly repetitive text
during unconditional generation under full self-conditioning (Full-SC), repeating words and phrases
(Figure~\ref{fig:basin}(a)). Gen-PPL, a commonly used likelihood-based quality proxy,
gives lower perplexity to more repetitive generations (Figure~\ref{fig:basin}(b)) and can conceal
this problem while biasing quality evaluation (\S\ref{sec:rep_problem}).
We find that the self-conditioning feedback in the Full-SC regime drives repetition by biasing the
denoising trajectory toward repetitive text along its dominant direction. With all other inference
settings fixed, we observe that repetition rises as the self-conditioning strength increases, while a
separately trained model without self-conditioning produces almost no repetition (Figure~\ref{fig:sc_strength}).
These observations point to self-conditioning as a causal driver of repetition (\S\ref{sec:sc_causal}). We further observe that
repetition rises during denoising and then levels off (Figure~\ref{fig:process}), suggesting an
attractor-like regime that favors repetitive text. This motivates a local fixed-point approximation of
the self-conditioning response underlying this attractor-like repetitive regime. After we decompose each
denoising update into a self-conditioning response toward the fixed point and an ordinary denoising
drive, theoretically, we predict a dominant direction by locally linearizing around this fixed point.
Empirically, we show that the scalar projection of a sample's self-conditioning feedback onto this
direction tracks its repetition level, while the self-conditioning feedback dynamics develop a dominant mode
aligned with it as the denoising trajectory settles into a repetitive regime.
These findings show that the self-conditioning feedback forms a repetition attractor that biases
denoising trajectories near the attractor toward its dominant direction, thereby producing repetitive
text (\S\ref{sec:mechanism}).

Unlike autoregressive degeneration, where repeated tokens feed back through the generated prefix and
accumulate along sequence order~\citep{holtzmancurious,xu2022learning}, self-conditioning
in continuous DLMs feeds clean-embedding predictions back across denoising steps before token
selection~\citep{chen2022analog,hu2026elf}. Existing remedies for autoregressive repetition typically
modify token-level decoding or training objectives, for example by changing sampling behavior or
penalizing repetitive continuations~\citep{welleckneural}. In continuous
DLMs, the repetition signal is formed upstream in the self-conditioning feedback of the clean-embedding prediction, so the intervention
should target the denoising dynamics rather than token selection. Motivated by this effectively
one-dimensional attractor, we propose Attractor-Contrast-Escape (ACE). ACE aims to escape the
attractor by removing the self-conditioning feedback component aligned with its dominant repetition direction. ACE first estimates a
repetition direction by contrasting the self-conditioning feedback from Full-SC denoising
paths trapped in repetition with that from paths relatively free of repetition, then subtracts it
from the feedback at each denoising step using a chosen steering strength
(Figure~\ref{fig:basin}(a)). This contrastive estimator is related to activation
steering~\citep{rimsky2024steering,shnaidman2025activation}, but ACE intervenes on self-conditioning
predictions rather than internal residual-stream activations. Further discussion of activation
steering and other related work is provided in App.~\ref{app:related}.
ACE precedes token selection without retraining or an auxiliary
model (\S\ref{sec:steer}).

At an appropriate steering strength, ACE reduces repetition in unconditional generation while retaining competitive quality. On ELF-B, it lowers mean REP-4 from $7.28\%$ to
$4.48\%$ while Grammar remains essentially unchanged, from $80.21$ to $80.57$
(Figure~\ref{fig:basin}(c)). More broadly, ACE reduces repetition across ELF-B/M/L while retaining
competitive results on several text-quality metrics beyond Gen-PPL, covering both individual-generation
quality and the quality of the generated set as a whole
(\S\ref{sec:main_results}).
ACE is a simple, theory-motivated intervention that reduces repetition and offers a competitive
repetition--quality trade-off against alternative self-conditioning methods (\S\ref{sec:alternatives}).
Further analysis examines this trade-off across steering choices
(\S\ref{sec:analysis}). ACE is robust to changes in inference configurations and random seeds for both
direction estimation and generation, and further generalizes to other
unconditional self-conditioned continuous diffusion language models (\S\ref{sec:gen}).

\section{Background}
\label{sec:background}

Continuous diffusion language models generate text by evolving a noisy state through a denoising trajectory in embedding space. At each step, a denoiser processes the current state and a sampler advances it. Self-conditioning~\citep{chen2022analog} augments this loop by feeding each step's clean-embedding prediction into the next step. Let $\bm{z}_{k-1}$ denote the current state at step $k$, $t_k$ its denoising time, and $\bm{x}_k$ the clean-embedding prediction at that step. The prediction is updated as
\begin{equation}
  \bm{x}_k
  = \bm{X}_\theta\bigl(\bm{z}_{k-1},t_k;\bm{x}_{k-1}\bigr),
  \label{eq:generic_sc}
\end{equation}
With either deterministic ODE or stochastic SDE updates, the joint denoising step is
\begin{equation}
  (\bm{z}_k,\bm{x}_k)
  = \textsc{Step}\bigl(\bm{X}_\theta,\bm{z}_{k-1},t_k,\bm{x}_{k-1}\bigr),
  \qquad k=1,\ldots,T,
  \label{eq:generic_step}
\end{equation}
where $T$ is the number of denoising steps and $\bm{x}_{k-1}$ is the self-conditioning signal.

Representative continuous DLMs with self-conditioning include Plaid~\citep{gulrajani2023likelihood}, LangFlow~\citep{chen2026langflow}, and ELF~\citep{hu2026elf}. More recently, ELF uses Flow Matching~\citep{lipman2022flow} for language generation, where a learned velocity field transports the state from Gaussian noise toward the clean text-embedding distribution. It follows this path from $t=0$ to $t=1$ and decodes tokens only at the terminal step.

For ELF, let $\Delta t_k=t_k-t_{k-1}$. The clean-embedding prediction parameterizes the velocity, and the deterministic ODE sampler uses the Euler update:
\begin{equation}
  \bm{v}_k
  = \frac{\bm{x}_k-\bm{z}_{k-1}}{1-t_k},
  \qquad
  \bm{z}_k
  = \bm{z}_{k-1}+\Delta t_k\bm{v}_k.
  \label{eq:elf_ode}
\end{equation}
The SDE-inspired sampler first re-injects Gaussian noise and evaluates the prediction at a slightly earlier, noisier time:
\begin{equation}
  \alpha_k=1-\gamma\Delta t_k,
  \qquad
  \bar{t}_k=\alpha_k t_k,
  \qquad
  \bar{\bm{z}}_{k-1}
  =\alpha_k\bm{z}_{k-1}+(1-\alpha_k)\boldsymbol{\epsilon}_k,
  \quad
  \boldsymbol{\epsilon}_k\sim\mathcal{N}(\bm{0},\bm{I}),
  \label{eq:elf_sde_perturb}
\end{equation}
and then updates the original state as
\begin{equation}
  \bm{z}_k
  =\bm{z}_{k-1}+\Delta t_k
    \frac{\bm{x}^{\mathrm{SDE}}_k-\bm{z}_{k-1}}{1-t_k},
  \label{eq:elf_sde_update}
\end{equation}
where $\bm{x}^{\mathrm{SDE}}_k$ is predicted from $(\bar{\bm{z}}_{k-1},\bar{t}_k)$. Generation starts from $\bm{z}_0\sim\mathcal{N}(\bm{0},\sigma^2\bm{I})$; $\sigma$ is the initial noise scale and $\gamma$ controls noise reinjection.

ELF also supports classifier-free guidance~\citep{ho2021classifier} with self-conditioning, linearly combining conditional and unconditional predictions with a guidance scale $w$ to strengthen the conditioning signal. During unconditional generation, ELF uses its previous clean-embedding prediction as the condition for CFG. Through training-time CFG, $\bm{X}_\theta$ takes $w$ as additional input and learns to predict the guided clean embedding directly, requiring only one forward pass per denoising step.

\section{Problem Setup and Motivation}
\label{sec:exp}

\begin{figure}[!t]
  \centering
  \includegraphics[width=0.98\linewidth]{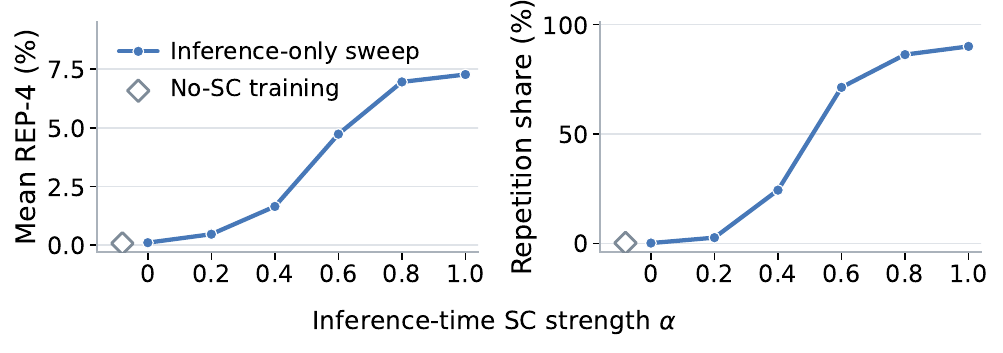}
  \caption{\textbf{Increasing self-conditioning strength increases repetition.} Mean REP-4 and
    repetition share are shown for the inference-only sweep and a model trained without self-conditioning.}
  \label{fig:sc_strength}
\end{figure}

\subsection{Repetitive generation and misleading Gen-PPL}
\label{sec:rep_problem}

Under the standard Full-SC setting, ELF can produce text that repeats the same words and phrases
many times (Figure~\ref{fig:basin}(a)). We examine unconditional generations from ELF-B, ELF-M,
and ELF-L and measure repetition using the $4$-gram self-repetition rate (REP-4;
\citealp{welleckneural}), defined as the fraction of $4$-grams that have already appeared in the same
sequence. We tokenize by whitespace. \emph{Mean REP-4} is the mean score across generated
samples and summarizes average repetition severity. \emph{Repetition share} is the fraction of samples
above a human-calibrated threshold and measures how widespread repetition is. We use the
$95$th-percentile REP-4 of BBC/XSum~\citep{narayan2018don} articles, $1.92\%$, as this threshold.
All three models exhibit repetition, with ELF-B showing the highest mean REP-4 and repetition
share (Table~\ref{tab:main}). Gen-PPL is perplexity under a pretrained language-model scorer,
with lower values interpreted as better generation. However, it decreases as repetition
increases (Figure~\ref{fig:basin}(b)), so it can conceal this repetition while biasing quality evaluation.

\subsection{Self-conditioning drives repetition}
\label{sec:sc_causal}
ELF uses self-conditioning to improve generation quality by refining the clean-embedding prediction
across denoising steps~\citep{chen2022analog,hu2026elf}. We ask whether this quality gain comes
with a cost in repetition. To test this, we scale the self-conditioning input as
$\bm{x}_{k-1}\mapsto\alpha\bm{x}_{k-1}$ and vary $\alpha$ from $0$ (off) to $1$ (Full-SC),
while keeping all other inference settings fixed. Both mean REP-4
and the repetition share increase with $\alpha$ (Figure~\ref{fig:sc_strength}). A separately trained
model without self-conditioning also produces almost no repetition, corroborating its causal role.
      
\begin{table}[!t]
  \centering
  \small
  \setlength{\tabcolsep}{2.0pt}
  \caption{\textbf{ACE reduces repetition across ELF model sizes while maintaining competitive text quality.}
    Full-SC is compared with ACE using the ELF-B direction $\bm d_B$ shared across sizes or model-specific directions $\bm d_M$ and $\bm d_L$.}
  \label{tab:main}
  \renewcommand{\arraystretch}{1.15}
  \newcommand{\tabletwostd}[1]{\scalebox{0.75}{$\scriptscriptstyle\pm #1$}}
  \begin{tabularx}{\linewidth}{@{}ll*{9}{>{\centering\arraybackslash}X}@{}}
    \toprule
    \raisebox{\dimexpr(\ht\strutbox+\dp\strutbox+3pt)/2\relax}[0pt][0pt]{Model} &
      \raisebox{\dimexpr(\ht\strutbox+\dp\strutbox+3pt)/2\relax}[0pt][0pt]{Method} &
      \metrichead{REP-4}{mean $\downarrow$} & \metrichead{Rep.}{share $\downarrow$} &
      \metrichead{Gen-}{PPL} & \metrichead{Grammar}{$\uparrow$} & \metrichead{Fluency}{$\uparrow$} &
      \metrichead{dist-$1$}{$\uparrow$} & \metrichead{dist-$3$}{$\uparrow$} &
      \metrichead{s-BLEU}{$\downarrow$} & \metrichead{MAUVE}{$\uparrow$} \\
    \midrule
    \multirow{2}{*}{ELF-B} & Full-SC & $7.28_{\tabletwostd{0.35}}$ & $90.00_{\tabletwostd{0.9}}$ & $20.24_{\tabletwostd{0.3}}$ & $80.21_{\tabletwostd{0.2}}$ & $\fluMainFullB$ & $43.24_{\tabletwostd{0.4}}$ & $88.19_{\tabletwostd{0.4}}$ & $29.03_{\tabletwostd{0.3}}$ & $23.63_{\tabletwostd{1.23}}$ \\
    & ACE ($\bm d_B$) & $4.48_{\tabletwostd{0.19}}$ & $74.02_{\tabletwostd{2.3}}$ & $23.66_{\tabletwostd{0.4}}$ & $80.57_{\tabletwostd{0.4}}$ & $\fluMainACESharedB$ & $46.75_{\tabletwostd{0.3}}$ & $91.93_{\tabletwostd{0.3}}$ & $29.61_{\tabletwostd{0.2}}$ & $19.74_{\tabletwostd{1.88}}$ \\
    \midrule
    \multirow{3}{*}{ELF-M} & Full-SC & $2.99_{\tabletwostd{0.15}}$ & $52.42_{\tabletwostd{2.1}}$ & $21.92_{\tabletwostd{0.3}}$ & $85.97_{\tabletwostd{0.2}}$ & $\fluMainFullM$ & $47.88_{\tabletwostd{0.3}}$ & $94.04_{\tabletwostd{0.2}}$ & $27.10_{\tabletwostd{0.2}}$ & $25.05_{\tabletwostd{1.31}}$ \\
    & ACE ($\bm d_B$) & $1.79_{\tabletwostd{0.07}}$ & $31.84_{\tabletwostd{1.4}}$ & $24.32_{\tabletwostd{0.4}}$ & $85.56_{\tabletwostd{0.2}}$ & $\fluMainACESharedM$ & $50.82_{\tabletwostd{0.2}}$ & $96.00_{\tabletwostd{0.1}}$ & $27.16_{\tabletwostd{0.2}}$ & $22.56_{\tabletwostd{2.58}}$ \\
    & ACE ($\bm d_M$) & $1.80_{\tabletwostd{0.08}}$ & $31.78_{\tabletwostd{1.4}}$ & $24.80_{\tabletwostd{0.4}}$ & $85.15_{\tabletwostd{0.2}}$ & $\fluMainACENativeM$ & $51.16_{\tabletwostd{0.2}}$ & $96.02_{\tabletwostd{0.1}}$ & $26.90_{\tabletwostd{0.2}}$ & $21.04_{\tabletwostd{1.39}}$ \\
    \midrule
    \multirow{3}{*}{ELF-L} & Full-SC & $2.50_{\tabletwostd{0.13}}$ & $43.61_{\tabletwostd{2.3}}$ & $23.70_{\tabletwostd{0.4}}$ & $85.83_{\tabletwostd{0.2}}$ & $\fluMainFullL$ & $50.14_{\tabletwostd{0.2}}$ & $94.94_{\tabletwostd{0.2}}$ & $27.27_{\tabletwostd{0.3}}$ & $23.04_{\tabletwostd{1.17}}$ \\
    & ACE ($\bm d_B$) & $1.58_{\tabletwostd{0.04}}$ & $26.60_{\tabletwostd{1.0}}$ & $25.56_{\tabletwostd{0.3}}$ & $85.60_{\tabletwostd{0.2}}$ & $\fluMainACESharedL$ & $52.44_{\tabletwostd{0.2}}$ & $96.45_{\tabletwostd{0.1}}$ & $27.51_{\tabletwostd{0.2}}$ & $19.81_{\tabletwostd{1.16}}$ \\
    & ACE ($\bm d_L$) & $1.56_{\tabletwostd{0.08}}$ & $26.46_{\tabletwostd{2.1}}$ & $25.48_{\tabletwostd{0.4}}$ & $85.63_{\tabletwostd{0.2}}$ & $\fluMainACENativeL$ & $52.27_{\tabletwostd{0.2}}$ & $96.49_{\tabletwostd{0.1}}$ & $27.57_{\tabletwostd{0.2}}$ & $19.63_{\tabletwostd{1.12}}$ \\
    \bottomrule
  \end{tabularx}
  \renewcommand{\arraystretch}{1}
\end{table}

\section{Mechanism and Method}
\label{sec:theory}

\suppressfloats[t]
\begin{figure}[t]
  \centering
  \includegraphics[width=0.98\linewidth]{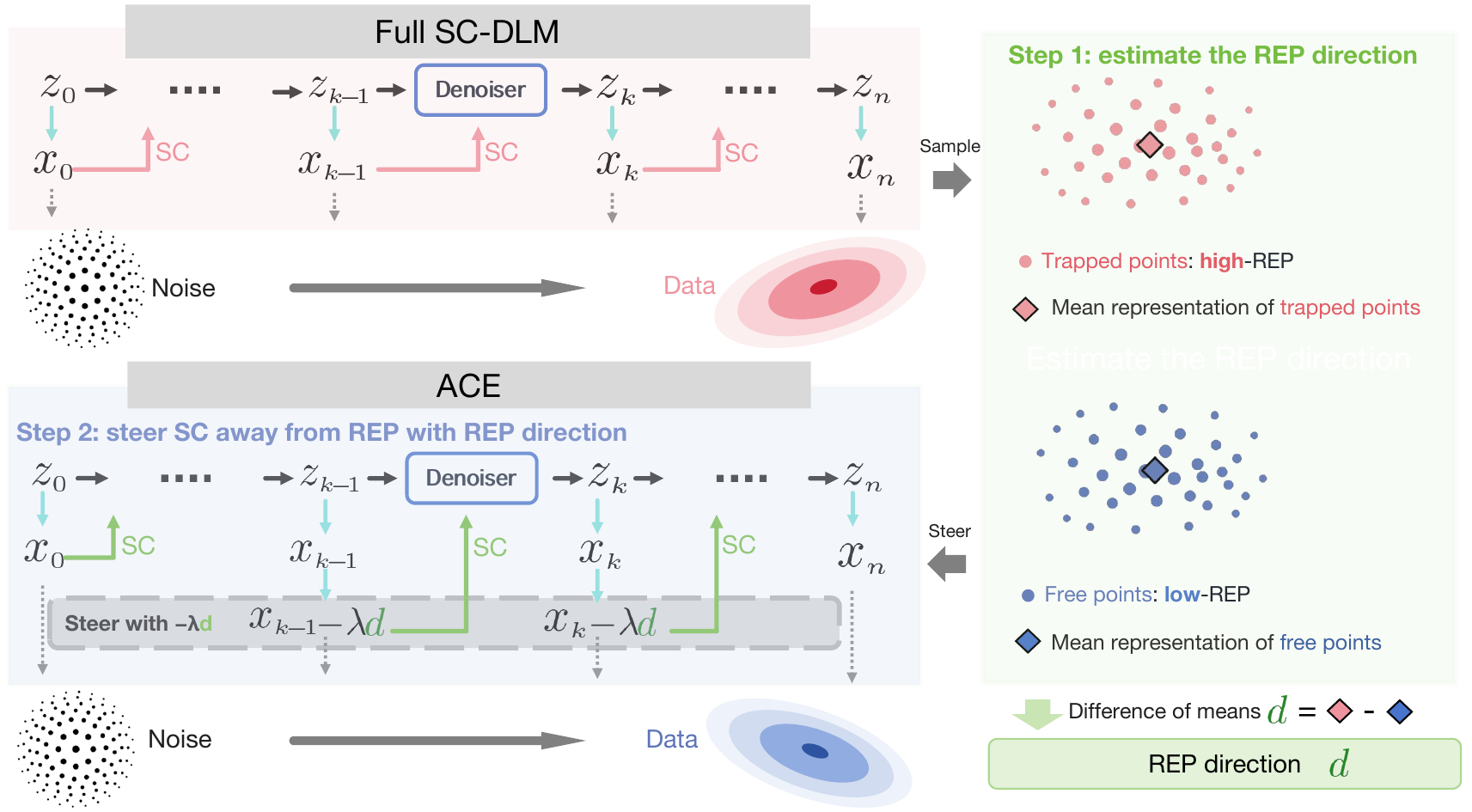}
  \caption{\textbf{ACE steers self-conditioning away from repetition.} Full-SC carries the
  self-conditioning feedback through denoising, whereas ACE first estimates a repetition direction
  from trapped (high-REP) and free (low-REP) trajectories, then subtracts it from feedback at each step.}
  \label{fig:method}
\end{figure}

\subsection{Mechanism: Self-conditioning forms a repetition attractor with one dominant direction}
\label{sec:mechanism}

Section~\ref{sec:exp} establishes that self-conditioning is a causal driver of repetition. To
understand how this self-conditioning feedback produces the defect, we decode intermediate states along a Full-SC
trajectory and track their repetition. The Full-SC curve in Figure~\ref{fig:process}(d) rises rapidly
and then levels off, suggesting an attractor-like regime. We model this regime locally with an
attracting fixed point.

\paragraph{A local model of the self-conditioning loop.}
Near this fixed point, we view self-conditioning as a feedback loop. The key idea is to separate the response to the fed-back
clean-embedding prediction from the remaining denoising drive. Let
$\bm{x}_k\in\mathbb{R}^{L\times e}$ denote the clean-embedding prediction at denoising step $k$,
and $\bm{u}_k\in\mathbb{R}^{e}$ its mean over $L$ positions. With denoising conditions locally fixed, we model
the update of $\bm{u}_k$ as a self-conditioning response plus a denoising drive:
\begin{equation}
  \bm{u}_k \;=\; \tfrac{1}{L}\textstyle\sum_{l=1}^{L}\bm{x}_k[l]\;\in\;\mathbb{R}^{e},
  \qquad
  \bm{u}_{k+1} \;=\;
  \underbrace{\bm{s}(\bm{u}_k)}_{\text{self-conditioning}}+\underbrace{\bm{f}_k}_{\text{denoising drive}}.
  \label{eq:scmap}
\end{equation}
Here $\bm{s}$ denotes the self-conditioning response defined around the fixed point, and
$\bm{f}_k$ is the remaining, locally fixed denoising drive. This decomposition
is defined in App.~\ref{app:theory}.

\paragraph{Theoretically, linearization predicts a dominant convergence direction within the repetition attractor.}
To understand how this fixed point shapes nearby denoising updates, we use a local fixed-point
approximation: we assume a smooth, contracting self-conditioning response and use first-order
linearization to compare the contraction rates along different directions.

\begin{assumption}[Local fixed point]\label{ass:attractor}
The self-conditioning map $\bm{s}$ has a fixed point $\bm{u}^\star$ and is
$C^1$ near it with a contracting Jacobian $\bm{J}=\mathrm{D}\bm{s}(\bm{u}^\star)$. For the analysis
we idealize $\bm{J}$ as symmetric, with orthonormal eigenvectors $\{\bm{v}_i\}$ and real eigenvalues
$1>\mu_1>\mu_2\ge\cdots\ge0$. The leading eigenvector $\bm{v}_1$ is the slowest-contracting
direction, and $\rho:=1-\mu_1\in(0,1)$ is its contraction rate.
\end{assumption}

The Jacobian describes how perturbations in the fed-back prediction propagate through the
self-conditioning loop.
Linearization gives (Lemmas~\ref{lem:linearize}--\ref{lem:decomp}, App.~\ref{app:theory})
\begin{equation}
  \Delta\bm{u}_k \;\approx\;
  \underbrace{\underset{\text{leading}}{\beta_k\,\bm{v}_1}\;+
  \underset{\text{decays}}{\bm{r}_k}}_{\text{dominant mode}}
  \;+\; \underbrace{\bm{f}_k}_{\text{denoising drive}}
  \label{eq:decomp}
\end{equation}
\noindent
Here $\bm{r}_k$ collects feedback components orthogonal to $\bm{v}_1$.
To characterize convergence, we hold the drive fixed
($\bm{f}_k\equiv\bm{f}$). The linearized loop converges to a limiting pooled clean-embedding
prediction $\bm{u}_\infty$ from an initial prediction $\bm{u}_0$. Lemma~\ref{lem:drivenfp} gives
its expression and derives the following eigenbasis expansion for deviations from the limit
(App.~\ref{app:linear}):
\begin{equation}
  \bm{u}_k-\bm{u}_\infty
  \;=\;\sum_i \mu_i^k
  \left\langle\bm{u}_0-\bm{u}_\infty,\bm{v}_i\right\rangle\bm{v}_i.
  \label{eq:mode-convergence}
\end{equation}
Under Assumption~\ref{ass:attractor}, the $\bm{v}_1$ component decays most slowly. If the initial
deviation has a nonzero component in this direction, it eventually dominates the remaining deviation,
yielding asymptotically one-dimensional convergence toward $\bm{u}_\infty$ along $\bm{v}_1$.

\begin{figure}[t]
  \centering
  \includegraphics[width=0.98\linewidth]{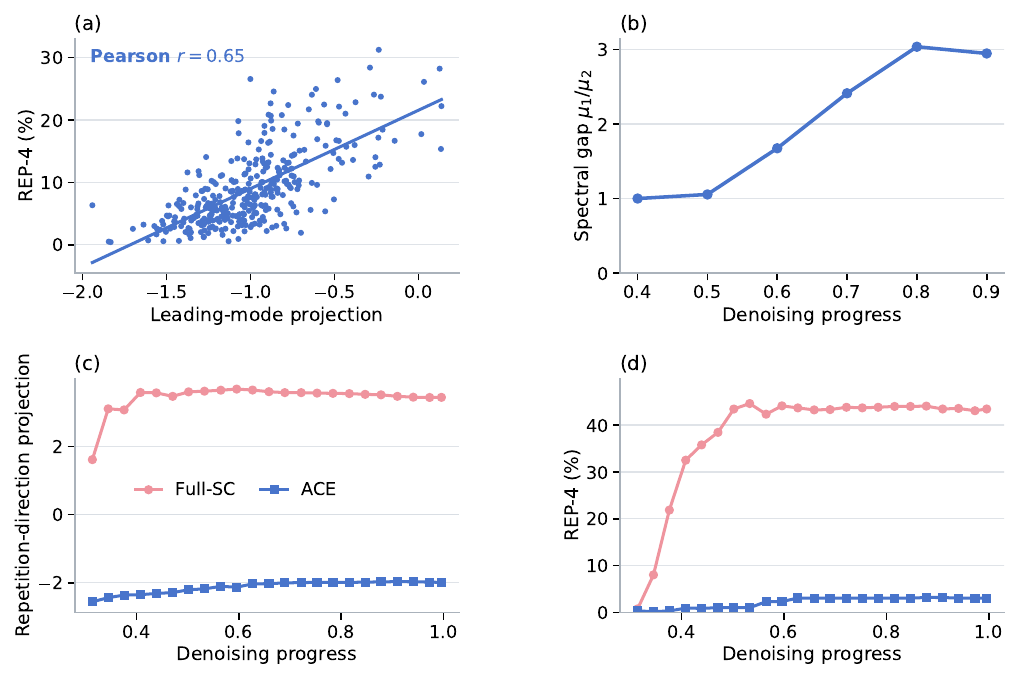}
  \caption{\textbf{A dominant self-conditioning mode accompanies repetition, while ACE redirects generation away from the estimated repetition direction.}
    (a) Across samples, REP-4 increases with the projection of the self-conditioning signal onto its leading mode; (b) the leading mode separates increasingly from the second mode as denoising progresses. (c) Full-SC gradually builds a stable positive projection along the estimated repetition direction associated with the leading mode, whereas ACE suppresses it.\linebreak (d) Accordingly, REP-4 rises and plateaus under Full-SC but stays low under ACE.}
  \label{fig:mechanism}
  \label{fig:process}
\end{figure}

\paragraph{Empirically, this dominant direction emerges with the attractor and predicts repetition.}
To characterize how this dominant direction relates to repetition and emerges during denoising, we
analyze self-conditioning feedback collected along Full-SC trajectories, comparing its projection
with repetition and tracking the local Jacobian spectrum. Across samples, the projection of
self-conditioning feedback onto $\bm{v}_1$ tracks repetition
(Fig.~\ref{fig:mechanism}(a)), while the spectral gap rises as the denoising trajectory enters the
repetitive regime and $\bm{v}_1$ becomes dominant (Fig.~\ref{fig:mechanism}(b), Tab.~\ref{tab:jac}).
To test whether the feedback--repetition association and widening spectral gap are consistent with a
local attractor, we hold the denoising conditions in Eq.~\ref{eq:scmap} fixed and repeatedly apply
the update. This brings the pooled prediction close to a fixed point
(Fig.~\ref{fig:pooled-fixedpoint}a).
Across denoising fractions, the leading feedback gain remains below one, while the dominant mode
accounts for a larger share of the local response in the later repetitive regime
(Fig.~\ref{fig:pooled-fixedpoint}b,c).
Together, these findings support a local attractor mechanism with effectively one-dimensional
repetition-related feedback. This structure motivates the directional intervention introduced in
the next subsection.
           
\subsection{Method: ACE redirects generation away from the dominant repetition direction}
\label{sec:steer}

The one-dimensional structure suggests a minimal intervention: estimate the repetition direction and
shift the self-conditioning feedback away from it, leaving orthogonal components unchanged.
We recover this direction by contrasting average self-conditioning feedback across
Full-SC trajectories. The resulting direction is fixed once and subtracted
from the self-conditioning feedback at each denoising step with strength $\lambda$. We call this procedure
\textbf{Attractor-Contrast-Escape (ACE)}.
Figure~\ref{fig:method} contrasts the unsteered Full-SC process with ACE's direction estimation
and steering stages.

Concretely, we estimate the repetition direction $\bm{d}$ once from $N$ Full-SC calibration samples. For each sample, we average
its self-conditioning feedback over denoising steps and sequence positions. Using a tail fraction
$q\in(0,1/2]$, we form trapped (high-REP) and free (low-REP) groups from the top and bottom $q$
fractions of calibration samples ranked by REP-4, respectively. The direction is the normalized difference between the two group
means. We assume that these groups differ mainly along the dominant mode,
while the mean contribution of the remaining modes is shared between the groups. Under this
separability assumption, the difference of means recovers the dominant mode $\bm{v}_1$ up to scale
(Prop.~\ref{prop:dmean}, App.~\ref{app:theory}). Empirically, $\bm{d}$ partially aligns with the leading
mode measured along denoising trajectories (App.~\ref{app:analysis}); ACE uses this trajectory-level
estimate for steering. During generation, we keep
$\bm{d}$ fixed and pass $\tilde{\bm{x}}_k=\bm{x}_k-\lambda\mathbf{1}_L\bm{d}^{\!\top}$ to the next
step.
See Algorithm~\ref{alg:steer} (App.~\ref{app:analysis}).

\begin{figure}[t]
  \centering
  \includegraphics[width=0.98\linewidth]{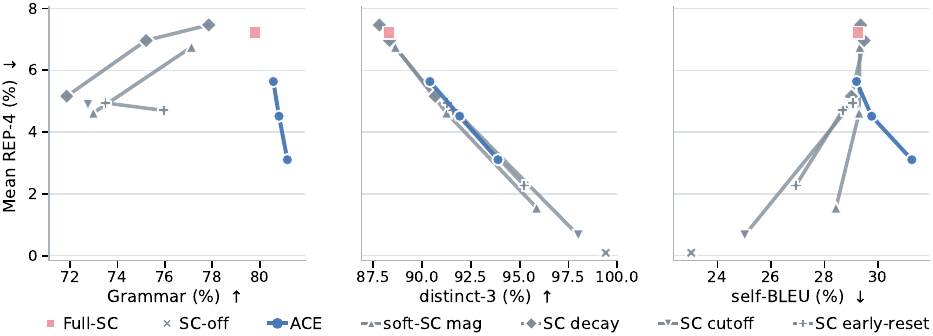}
  \vspace{-5pt}
  \caption{\textbf{Self-conditioning interventions exhibit repetition--quality trade-offs, with ACE
    achieving a favorable balance.} The panels compare repetition with complementary quality and
    diversity measures across methods; lines connect tested settings within each method. Only points
    within the displayed ranges are shown; full sweeps and additional metrics are reported in
    App.~\ref{app:softsc}.}
  \label{fig:alternatives-tradeoff}
\end{figure}

\section{Experiments}
\label{sec:experiments}

\subsection{Main results: reducing repetition while preserving quality}
\label{sec:main_results}

We evaluate whether ACE reduces repetition while maintaining generation quality. We report mean
REP-4 and repetition share to measure repetition, and
assess generation quality at two levels. For individual generations, Grammar and Fluency assess
linguistic quality, while distinct-$1$ and distinct-$3$ measure within-text diversity; for the
generated set as a whole, self-BLEU measures cross-sample diversity and MAUVE measures similarity to
the reference distribution. Evaluation details are provided in App.~\ref{app:evaluation}.
All results use a
common $64$-step SDE setting with initial noise scale
$\sigma{=}2.0$, churn $\gamma{=}1.0$, and SC-CFG scale $w{=}3$. We choose ACE's calibration
size $N$ and tail fraction $q$ by assessing the stability of the estimated direction, using
adjacent-size cosine across $N$ and cross-seed agreement across $q$. As $N$ grows from
$256$ to $4096$, the mean adjacent-size absolute cosine across the plotted curves rises from
$0.917$ to $0.995$, indicating that the direction estimates stabilize at $N{=}4096$
(App. Fig.~\ref{fig:ace-calibration-stability}). Among the tested tail fractions
$q\in\{1/8,1/4,1/2\}$, $q{=}1/4$ achieves the highest cross-seed agreement
(App. Fig.~\ref{fig:multiseed-stability-quality}a). We then select steering strength $\lambda$
based on the repetition--quality trade-off; the strength-response analysis supports
$\lambda{=}1/2$ as a quality-focused choice (Figure~\ref{fig:ace-strength-main}). We therefore use
$N{=}4096$, $q{=}1/4$, and $\lambda{=}1/2$ for the main results. Metrics aggregate
five held-out generation seeds with $1{,}024$ samples per seed. Table~\ref{tab:main} summarizes the
main comparison. On ELF-B, ACE reduces mean REP-4
from $7.28\%$ to $4.48\%$ and repetition share from $90.00\%$ to $74.02\%$. Gen-PPL rises after
steering as repetition is reduced, while the other quality measures remain competitive. Panels~(c,d) of
Figure~\ref{fig:process} show this
pattern along one ELF-B trajectory: Full-SC's feedback projection onto $\bm{d}$ remains positive as
REP-4 rises, whereas ACE shifts it to the opposite side and reduces repetition. Paired qualitative
examples illustrating this contrast are provided in App.~\ref{app:qualitative}.
This quality preservation is consistent with the directional intervention:
ACE shifts feedback along $\bm{d}$ while leaving orthogonal components unchanged at each intervention
(App.~\ref{app:theory}). As a complementary check, applying the same direction
as a training-time penalty further reduces repetition beyond continued training alone; the training
setup and full results are reported in App.~\ref{app:training}. Table~\ref{tab:main} shows that ACE
remains effective on ELF-M and ELF-L: using model-specific directions $\bm d_M$ and $\bm d_L$, it
reduces repetition while retaining competitive quality. It also shows that the ELF-B direction
$\bm d_B$ can be shared with both models, where ACE likewise reduces repetition while retaining
competitive quality. Fig.~\ref{fig:ace-calibration-stability}(c) further shows that the directions
estimated separately for the three model sizes are aligned. This suggests a systematic repetition bias across unconditional continuous
diffusion language models. Transfer across settings and seeds is examined in \S\ref{sec:gen}.

\begin{figure}[t]
  \centering
  \includegraphics[width=0.98\linewidth]{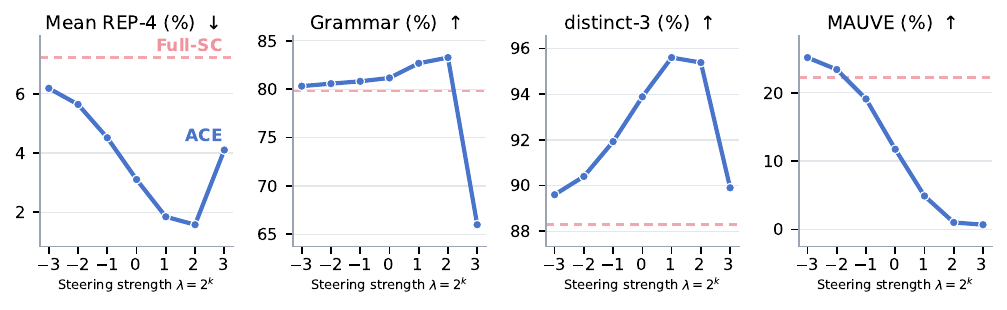}
  \caption{\textbf{A moderate steering strength gives a favorable repetition--quality trade-off.}
    The four panels compare repetition, single-generation quality, and sample-pool quality
    against Full-SC across steering strengths. Further strength and direction analyses are
    provided in App.~\ref{app:analysis}.}
  \label{fig:ace-strength-main}
\end{figure}

\subsection{Comparison with alternative repetition interventions}
\label{sec:alternatives}

Several alternative interventions modify the self-conditioning feedback and can also reduce repetition.
SC-off disables the self-conditioning feedback during inference, while soft-SC mag rescales its magnitude.
SC decay gradually attenuates the feedback over denoising, and SC cutoff disables the
feedback for all samples after a fixed denoising fraction. SC early-reset checks
for repetition midway through denoising, then sets the self-conditioning feedback to zero for the
remaining steps of any sample that has already begun repeating. We include the Full-SC sampler for
reference.
We compare the methods under the same held-out generation seed, with $1{,}024$ direct
generations per setting.
Figure~\ref{fig:alternatives-tradeoff} plots mean REP-4 against Grammar, distinct-$3$, and
MAUVE; the full parameter grid and complementary metrics are reported in App.~\ref{app:softsc}.
These results show a trade-off: interventions that suppress feedback
more broadly can reduce repetition sharply but may compromise generation quality, while softer changes
can leave more repetition. \mbox{Notably,} ACE employs a simple theory-guided intervention to suppress repetition relative to Full-SC\linebreak
while retaining competitive quality, thereby yielding a favorable repetition--quality trade-off.

\subsection{Analysis of ACE's steering strength and direction}
\label{sec:analysis}

With the estimated repetition direction fixed, we vary the steering strength $\lambda$ and measure
repetition and quality. Figure~\ref{fig:ace-strength-main} shows that increasing $\lambda$ initially
suppresses repetition while slightly improving single-generation quality; sample-pool quality remains
stable or slightly improves at weak strengths. At larger strengths, however, further repetition
reduction can come at the cost of sample-pool quality, even when single-generation quality improves.
When the steering strength becomes too large,
single-generation quality also deteriorates and repetition rebounds, suggesting over-steering. Across the
range of steering strengths, we compare direction estimators motivated by our repetition-attractor theory, including ACE,
Logistic, LDA, and the Jacobian mode. Logistic and LDA provide discriminative estimates of the theory's
repetition direction, whereas the Jacobian mode extracts the dominant local feedback mode predicted by
the same analysis. For comparison, we also evaluate alternatives that do not use the
repetition-attractor hypothesis. Top-PC extracts the leading singular direction of pooled feedback,
using REP-4 to select its calibration subset and orient its sign. Label-permuted directions use
the same REP-based estimator with the repetition labels shuffled, and random directions ignore the
feedback structure. Figure~\ref{fig:multiseed-stability-quality} and Table~\ref{tab:heldout-controls} show that the
reduction depends on alignment with the repetition contrast: at suitable steering strengths,
theory-motivated estimators reduce repetition while retaining competitive quality, whereas
theory-agnostic controls do not reliably reproduce this behavior. ACE
provides a simple, stable, and effective implementation of this principle; further details appear in
App.~\ref{app:analysis}.

\subsection{Robustness and Generalization}
\label{sec:gen}

We first test the robustness of ACE to changes in inference configurations and random seeds within ELF. Keeping the
shared ELF-B direction $\bm d_B$ fixed, we vary the number of denoising steps, SC-CFG guidance scale,
initial noise scale $\sigma$, sampler, and churn $\gamma$, and evaluate selected settings across
independent held-out generation seeds. We additionally re-estimate $\bm d_B$ from five independent calibration seeds to test
the stability of the direction itself. ACE continues to reduce repetition while retaining
competitive quality across these settings (Figure~\ref{fig:knobs} and Table~\ref{tab:crosssize}).
The repetition gap is especially pronounced at larger denoising-step counts and guidance scales, where
Full-SC trajectories are driven deeper into the repetition basin, while ACE keeps them away from it
(Figure~\ref{fig:knobs}). Table~\ref{tab:main} already shows that ACE generalizes across ELF model sizes.
We next ask whether the same intervention extends beyond the ELF family. LangFlow~\citep{chen2026langflow}
and Plaid~\citep{gulrajani2023likelihood} are self-conditioned continuous diffusion language models
with distinct text representations. Using a direction estimated separately for each model, ACE
reduces repetition while retaining competitive quality on most reported measures
(Table~\ref{tab:crossmodel}). Detailed settings and steering-strength trade-offs
are reported in App.~\ref{app:robust-model}. Together,
these results show that ACE is robust and can generalize across models.

\begin{figure}[t]
  \centering
  \includegraphics[width=0.84\linewidth]{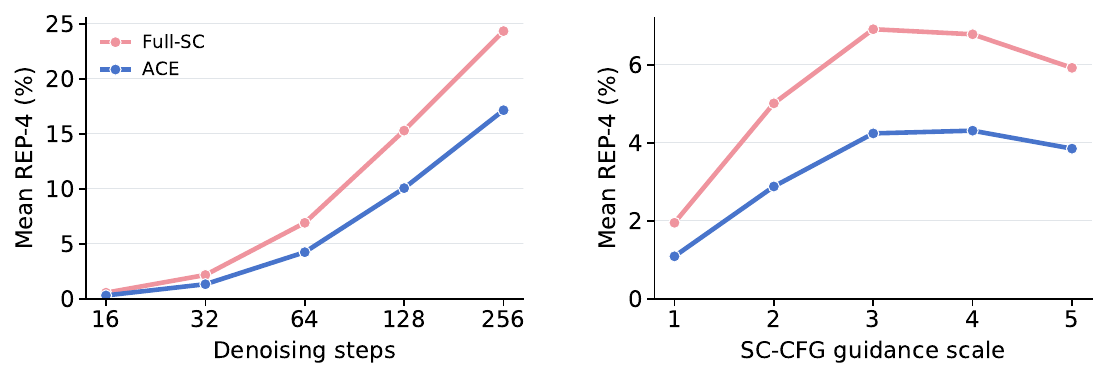}
  \caption{\textbf{ACE suppresses repetition across denoising steps and guidance scales.}
    Mean REP-4 for Full-SC (pink) and ACE (blue) as the denoising steps and SC-CFG guidance
    scale vary.}
  \label{fig:knobs}
\end{figure}

\begin{table}[t]
  \centering
  \footnotesize
  \setlength{\tabcolsep}{2.0pt}
  \caption{\textbf{ACE reduces repetition across other self-conditioned continuous DLMs.}
    Full-SC is compared with ACE on LangFlow and Plaid using model-specific directions.}
  \label{tab:crossmodel}
  \renewcommand{\arraystretch}{1.15}
  \begin{tabularx}{\linewidth}{@{}ll*{8}{>{\centering\arraybackslash}X}@{}}
    \toprule
    \raisebox{\dimexpr(\ht\strutbox+\dp\strutbox+3pt)/2\relax}[0pt][0pt]{Model} &
      \raisebox{\dimexpr(\ht\strutbox+\dp\strutbox+3pt)/2\relax}[0pt][0pt]{Method} &
      \metrichead{REP-4}{mean $\downarrow$} & \metrichead{Rep.}{share $\downarrow$} &
      \metrichead{Grammar}{$\uparrow$} & \metrichead{Fluency}{$\uparrow$} &
      \metrichead{dist-$1$}{$\uparrow$} & \metrichead{dist-$3$}{$\uparrow$} &
      \metrichead{s-BLEU}{$\downarrow$} & \metrichead{MAUVE}{$\uparrow$} \\
    \midrule
    \multirow{2}{*}{LangFlow} & Full-SC & $3.98$ & $35.74$ & $61.56$ & $1.00$ &
      $52.31$ & $93.76$ & $18.90$ & $14.44$ \\
    & ACE & $2.72$ & $25.39$ & $61.67$ & $1.01$ &
      $53.65$ & $95.44$ & $20.03$ & $15.59$ \\
    \midrule
    \multirow{2}{*}{Plaid} & Full-SC & $3.95$ & $30.30$ & $89.48$ & $2.04$ &
      $51.09$ & $93.93$ & $19.68$ & $15.52$ \\
    & ACE & $2.82$ & $25.20$ & $89.75$ & $2.12$ &
      $51.01$ & $95.22$ & $20.95$ & $14.73$ \\
  \bottomrule
  \end{tabularx}
  \renewcommand{\arraystretch}{1}
\end{table}

\section{Conclusion}
\label{sec:conclusion}

We find that self-conditioned continuous diffusion language models can produce substantial
repetition in unconditional generation, a defect that Gen-PPL can conceal.
We trace this repetition to an effectively one-dimensional contractive attractor in the self-conditioning
loop. Motivated by this mechanism, we propose Attractor-Contrast-Escape (ACE), a training-free
inference-time intervention that estimates a repetition direction by contrasting denoising paths
trapped in repetition with paths relatively free of repetition, then subtracts it from the
self-conditioning feedback during denoising. At appropriate steering strengths, ACE reduces repetition
while retaining competitive text quality across ELF model sizes, inference configurations, and random
seeds. It also mitigates repetition in other self-conditioned continuous DLMs.
These results identify self-conditioning feedback as a source of repetition and demonstrate that targeted
intervention during denoising can mitigate it.

\section*{AI use statement}

In this work, we used generative AI tools to assist with literature discovery; research ideation, experimental design, and code implementation and debugging; interpretation of experimental results; mathematical reasoning and review of proof arguments; and drafting and editing portions of the manuscript. The authors reviewed the AI-assisted contributions and take responsibility for the final text, analyses, and claims.

\bibliography{paper}
\bibliographystyle{iclr2027_conference}

\appendix
\section{Evaluation Protocol and Metric Definitions}
\label{app:evaluation}

Unless specified otherwise, we evaluate each model on $1{,}024$ samples generated with each held-out
seed, using directions estimated from separate calibration samples.
For multi-seed evaluations, we report mean scores over five seeds; accompanying
$\pm$ values denote sample standard deviations across seeds.
The canonical ELF setting is $64$ SDE denoising steps with initial
noise scale $\sigma{=}2.0$, churn $\gamma{=}1.0$, and SC-CFG scale $w{=}3$.
We consider a generated sample repetitive if its REP-4 exceeds the $95$th-percentile REP-4 of
BBC/XSum articles ($1.92\%$). Table~\ref{tab:eval_metrics} summarizes the implementation and aggregation of every metric
reported in
the paper.  REP-4, repetition share, Grammar, distinct-$1$, distinct-$3$, self-BLEU, and MAUVE are
reported in percentage points in the tables; Gen-PPL is left on its native scale, and Fluency is reported
on its $1$--$5$ judge scale.  Following the reference-free unconditional-generation evaluation setup of
\citet{patel2025insertion}, we adapt its prompt and fluency rubric to GPT-4o. The box below shows the
full judging prompt, with \texttt{\{generation\}} replaced by the generated text. We extract the
integer following \texttt{[RESULT]} as the score. GPT-4o uses temperature $0$, top-$p$ $0.9$, and a
maximum completion length of $1{,}024$ tokens. All metric implementations and the raw per-sample
Fluency records are retained with the corresponding evaluation manifests.

\begin{mdframed}[style=examplebox,frametitle={Fluency evaluation prompt},nobreak=true]
\small
\setlength{\parindent}{0pt}
\setlength{\parskip}{4pt}
\textbf{System message}

You are a fair judge assistant tasked with providing clear, objective feedback based on specific
criteria, ensuring each assessment reflects the absolute standards set for performance.

\medskip
\hrule
\medskip
\textbf{User message}

\textbf{\#\#\#Task Description:}

An unconditional generation to evaluate, and a score rubric representing an evaluation criteria are given.
\begin{enumerate}[label=\arabic*.,leftmargin=1.5em,itemsep=2pt,parsep=0pt,topsep=2pt]
  \item Write a detailed feedback that assesses the quality of the generation strictly based on the
    given score rubric, not evaluating in general.
  \item After writing a feedback, write a score that is an integer between 1 and 5. You should refer
    to the score rubric.
  \item The output format should look as follows: ``(write a feedback for criteria) [RESULT]
    (an integer number between 1 and 5)''.
  \item Please do not generate any other opening, closing, or explanations.
\end{enumerate}

\textbf{Generation to evaluate:}

\texttt{\{generation\}}

\textbf{Score Rubrics:}

Fluency (Is the text fluent and easy to read?)

\textbf{Score of 1:} The text is disjointed and lacks fluency, making it hard to follow.

\textbf{Score of 2:} The text has limited fluency with frequent awkward phrasing.

\textbf{Score of 3:} The text is moderately fluent, with some awkward phrasing but generally easy to follow.

\textbf{Score of 4:} The text is fluent with smooth transitions and rare awkward phrases.

\textbf{Score of 5:} The text is highly fluent, with natural and smooth expression throughout.

\textbf{Feedback:}
\end{mdframed}

\begin{table*}[t]
  \centering
  \small
  \setlength{\tabcolsep}{4pt}
  \caption{Evaluation metrics and their aggregation protocols.}
  \label{tab:eval_metrics}
  \begingroup
  \renewcommand{\arraystretch}{1.15}
  \begin{tabular}{@{}p{0.15\textwidth}p{0.36\textwidth}p{0.44\textwidth}@{}}
    \toprule
    Metric & Implementation & Aggregation and interpretation \\
    \midrule
    REP-4 and repetition share
      & Whitespace-tokenized $4$-grams; REP-4 is the fraction of occurrences that duplicate an earlier
        $4$-gram.
      & We report the mean REP-4 over all generated samples and the fraction of samples above the
        $1.92\%$ human-calibrated threshold. \\
    Gen-PPL
      & GPT-2 Large~\citep{radford2019language} scores each generated token with a $1{,}024$-token
        evaluation context.
      & The exponential of the corpus-level mean token negative log-likelihood.  Lower values mean that
        the scorer finds the sample corpus more predictable. \\
    Grammar
      & A RoBERTa-large classifier trained on CoLA~\citep{warstadt2019neural,krishna2020reformulating}, loaded from the
        \texttt{roberta-large-cola-}\allowbreak\texttt{krishna2020} checkpoint.
      & We split each generation into sentences of at least three words, truncate each sentence to the
        classifier's 512-token input limit (including special tokens), score its probability of
        acceptability, average within each text, and then average over texts. \\
    Fluency
      & GPT-4o directly assesses the fluency of each unconditional generation.
      & We average the parsed scores over generations and then held-out seeds; scores range from 1 to 5,
        with higher values indicating greater fluency. \\
    distinct-$1$ and distinct-$3$
      & For each text, unique whitespace-token $n$-grams divided by total $n$-grams.
      & We first average the per-text ratios and then average across the sample pool; larger values indicate
        greater within-text diversity. \\
    self-BLEU
      & Smoothed BLEU-4 against remaining texts, using whitespace tokens and up to $1{,}024$ leave-one-out
        references.
      & We average the per-text scores over the pool.  Because this is a similarity measure, lower values
        indicate less cross-sample similarity. \\
    MAUVE
      & GPT-2-Large features and MAUVE divergence-frontier estimator~\citep{pillutla2021mauve};
        $1{,}024$ OpenWebText references; generation cap of $1{,}024$ tokens.
      & We report the resulting distributional score on a $0$--$100$ scale as an auxiliary comparison
        with the chosen reference corpus. \\
    \bottomrule
  \end{tabular}
  \endgroup
\end{table*}
 % evaluation protocol and metric definitions
\section{Formal analysis and empirical validation of the repetition attractor}\label{app:theory}

This appendix formalizes the local picture developed in \S\ref{sec:theory}. We first construct a local
model of the self-conditioning loop, linearize it near a fixed point, separate the self-conditioning
response from the denoising drive, and characterize the resulting driven fixed point. We also extend
the fixed-point result locally to the nonlinear map and discuss its quasi-static interpretation when
the local coefficients vary slowly (\S\ref{app:linear}). Under a separability assumption, we then show
how the difference between mean self-conditioning feedback from trapped and free Full-SC trajectories
recovers the dominant direction used to steer repetition; empirical alignment is approximate
(\S\ref{app:dmean}). Finally, we test numerical convergence of the frozen pooled map and estimate the
local Jacobian on the trained model to assess contraction, mode dominance, and separation of the
denoising drive from the self-conditioning response (\S\ref{app:empirical}). These conclusions are
local: the linear analysis assumes a symmetric Jacobian and treats the latent trajectory as exogenous,
with the drive frozen when characterizing the fixed point. Direction recovery additionally relies on
separability; the analysis does not provide a global decomposition of the denoising trajectory.

\newtheorem{lemma}{Lemma}

\subsection{Local model of the self-conditioning loop}\label{app:linear}

\paragraph{Local feedback map.}
To analyze the feedback at the sequence level, we pool the clean-embedding predictions in
\eqref{eq:generic_sc} and use Assumption~\ref{ass:attractor}. Let
$\overline{(\cdot)}=\tfrac1L\sum_{l=1}^{L}(\cdot)[l]$ denote position pooling and define the
pooled model map
\begin{equation}
  \overline{\bm{X}}_\theta(\bm{z},t;\bm{u})
  := \frac1L\sum_{l=1}^{L}
  \bm{X}_\theta(\bm{z},t;\mathbf{1}_L\bm{u}^{\!\top})[l]
  \in\mathbb{R}^{e},
  \label{eq:pooled_model}
\end{equation}
where $\mathbf{1}_L\bm{u}^{\!\top}$ denotes the self-conditioning feedback broadcast to all positions.
Under this pooled approximation, the feedback update at step $k+1$ is:
\begin{equation}
  \bm{u}_{k+1}
  = g_{k+1}(\bm{u}_k)
  := \overline{\bm{X}}_\theta(\bm{z}_k,t_{k+1};\bm{u}_k).
  \label{eq:gmodel}
\end{equation}
The subscript on $g_{k+1}$ records the changing state and denoising time. In the local linearized model,
we treat the denoising conditions as exogenous to the feedback input and use a common Jacobian, while
allowing the drive to vary across steps. The fixed-point result below additionally assumes a constant
drive. Because
$\bm{u}_k$ enters through the fed-back channel, we expand the pooled map about the fixed point
$\bm{u}^\star$:
\begin{equation}
  \overline{\bm{X}}_\theta(\bm{z}_k,t_{k+1};\bm{u}_k)
  \;=\;\underbrace{\overline{\bm{X}}_\theta(\bm{z}_k,t_{k+1};\bm{u}^\star)}_{\text{zeroth order}}
  \;+\;\underbrace{\bm{J}\,(\bm{u}_k-\bm{u}^\star)}_{\text{first order}}
  \;+\;o(\lVert\bm{u}_k-\bm{u}^\star\rVert).
  \label{eq:taylor}
\end{equation}
The three objects used throughout are the terms of \eqref{eq:taylor}: the
\emph{denoising drive}
$\bm{f}_k:=\overline{\bm{X}}_\theta(\bm{z}_k,t_{k+1};\bm{u}^\star)-\bm{u}^\star$,
which is independent of $\bm{u}_k$ under the frozen local model, the self-conditioning Jacobian
$\bm{J}:=\partial\overline{\bm{X}}_\theta/\partial\bm{u}\big|_{\bm{u}^\star}$, and the
self-conditioning map
$\bm{s}(\bm{u}):=\bm{u}^\star+\bm{J}(\bm{u}-\bm{u}^\star)+o(\lVert\bm{u}-\bm{u}^\star\rVert)$,
which fixes $\bm{u}^\star$ by construction. Collecting them gives the stepwise split, with the
local-map index suppressed:
\begin{equation}
  g(\bm{u}_k)\;=\;\bm{s}(\bm{u}_k)+\bm{f}_k .
  \label{eq:gsplit}
\end{equation}
Within this local model, the latent trajectory $\{\bm{z}_k\}$ is treated as exogenous, so $\bm{f}_k$
acts as a $\bm{u}_k$-independent forcing. This decoupling approximation, together with the local
Jacobian assumptions, is evaluated empirically in App.~\ref{app:empirical}.

Let $\bm{a}_k=\bm{u}_k-\bm{u}^\star$ be the residual, with coordinates
$a^{(i)}_k=\langle\bm{a}_k,\bm{v}_i\rangle$ in the eigenbasis of Assumption~\ref{ass:attractor};
by the symmetric-$\bm{J}$ idealization $\bm{v}_1$ is both a left and right eigenvector, so the
$\bm{v}_1$-coordinate decouples below.

\begin{lemma}[Linearized loop]\label{lem:linearize}
Near the fixed point the residual obeys the linear recursion
\begin{equation}
  \bm{a}_{k+1}\;=\;\bm{J}\,\bm{a}_k+\bm{f}_k ,
  \label{eq:residrec}
\end{equation}
the residual form of the loop \eqref{eq:scmap} near $\bm{u}^\star$.
\end{lemma}
\begin{proof}
Subtract $\bm{u}^\star$ from the split \eqref{eq:gsplit} and use
$\bm{s}(\bm{u}_k)=\bm{u}^\star+\bm{J}\bm{a}_k+o(\lVert\bm{a}_k\rVert)$ from \eqref{eq:taylor}:
$\bm{a}_{k+1}=\bm{J}\bm{a}_k+\bm{f}_k+o(\lVert\bm{a}_k\rVert)$, which is \eqref{eq:residrec} once
the higher-order remainder is dropped (exact for the idealized linear model).
\end{proof}

\begin{lemma}[Per-step decomposition]\label{lem:decomp}
Within the linearized model, the one-step change splits as
\begin{equation}
  \Delta\bm{u}_k\;=\;\beta_k\,\bm{v}_1\;+\;\bm{r}_k\;+\;\bm{f}_k ,
  \tag{\ref{eq:decomp}}
\end{equation}
with repetition coefficient
$\beta_k=(\mu_1-1)\,a^{(1)}_k$ and remainder $\bm{r}_k=\sum_{i\ge2}(\mu_i-1)\,a^{(i)}_k\bm{v}_i$.
The remainder collects self-conditioning modes orthogonal to $\bm{v}_1$, which contract faster
($\mu_i\le\mu_1$ for $i\ge2$); the denoising drive $\bm{f}_k$ remains a separate term.
\end{lemma}
\begin{proof}
From Lemma~\ref{lem:linearize},
$\Delta\bm{u}_k=\bm{u}_{k+1}-\bm{u}_k=\bm{a}_{k+1}-\bm{a}_k=(\bm{J}-\bm{I})\bm{a}_k+\bm{f}_k$.
Expanding the residual in the orthonormal eigenbasis (Assumption~\ref{ass:attractor}),
$\bm{a}_k=\sum_i a^{(i)}_k\bm{v}_i$, gives
$(\bm{J}-\bm{I})\bm{a}_k=\sum_i(\mu_i-1)a^{(i)}_k\bm{v}_i$; separating the $i{=}1$ term as
$\beta_k\bm{v}_1$ from the rest as $\bm{r}_k$ yields the stated identity. Projecting
\eqref{eq:residrec} on $\bm{v}_i$ gives the scalar recursion
$a^{(i)}_{k+1}=\mu_i a^{(i)}_k+f_k^{(i)}$. Its homogeneous transient decays as $\mu_i^k$, fastest for
the smallest $\mu_i$; consequently, the orthogonal transients ($\mu_i\le\mu_1$, $i\ge2$) die out at
least as fast as the $\bm{v}_1$ transient. When the drive is frozen, each coordinate settles at the
driven offset $f^{(i)}/(1-\mu_i)$, as made explicit in Lemma~\ref{lem:drivenfp}. Thus $\bm{v}_1$
is the slowest-decaying self-conditioning mode, while the drive contribution remains separate.
\end{proof}

\begin{lemma}[Where the driven loop settles]\label{lem:drivenfp}
With the drive frozen at $\bm{f}_k\equiv\bm{f}$, the recursion \eqref{eq:residrec} converges
geometrically (at rate $\mu_1$) to the unique fixed point $\bm{a}_\infty=(\bm{I}-\bm{J})^{-1}\bm{f}$,
i.e.
\begin{equation}
  \bm{u}_\infty=\bm{u}^\star+(\bm{I}-\bm{J})^{-1}\bm{f}.
  \label{eq:driven-fp}
\end{equation}
Its distance from $\bm{u}^\star$ along the repetition axis is
$\lvert a^{(1)}_\infty\rvert=\lvert f^{(1)}\rvert/(1-\mu_1)=\lvert f^{(1)}\rvert/\rho$ with
$f^{(1)}=\langle\bm{f},\bm{v}_1\rangle$. For fixed $\rho>0$, a small but nonzero $f^{(1)}$ yields a
correspondingly small but nonzero offset along $\bm{v}_1$, rather than collapse along this mode. The
empirical size of this axial component is evaluated in App.~\ref{app:empirical}. For fixed
$f^{(1)}$, this offset decreases as $\rho$ increases; a near-marginal mode permits a larger offset.
The repetition interpretation is supported by the
empirical association between leading-mode projection and REP (Fig.~\ref{fig:mechanism}(a)).
\end{lemma}
\begin{proof}
$\bm{J}$ is contracting, so its spectral radius $\mu_1<1$ and $\bm{I}-\bm{J}$ is invertible
(eigenvalues $1-\mu_i>0$); the Neumann series $\sum_{j\ge0}\bm{J}^j=(\bm{I}-\bm{J})^{-1}$
converges. Unrolling \eqref{eq:residrec} with $\bm{f}_k\equiv\bm{f}$ from $\bm{a}_0$,
\begin{equation}
  \bm{a}_k=\bm{J}^k\bm{a}_0+\sum_{j=0}^{k-1}\bm{J}^j\bm{f}\;\xrightarrow[k\to\infty]{}\;(\bm{I}-\bm{J})^{-1}\bm{f},
  \label{eq:unroll}
\end{equation}
the transient $\bm{J}^k\bm{a}_0$ vanishing at rate $\mu_1^k$. Equivalently $\bm{a}_\infty$
solves $(\bm{I}-\bm{J})\bm{a}_\infty=\bm{f}$; projecting on the unit eigenvector $\bm{v}_1$,
$(1-\mu_1)a^{(1)}_\infty=f^{(1)}$, so $\lvert a^{(1)}_\infty\rvert=\lvert f^{(1)}\rvert/(1-\mu_1)$.
The map $\rho\mapsto\lvert f^{(1)}\rvert/\rho$ is decreasing on $(0,1)$, so a larger $\rho$ shrinks
the leading-mode offset.

At the fixed point, $\bm{a}_\infty=\bm{J}\bm{a}_\infty+\bm{f}$. Subtracting this relation from
\eqref{eq:residrec} and iterating gives
$\bm{a}_k-\bm{a}_\infty=\bm{J}^k(\bm{a}_0-\bm{a}_\infty)$. Since
$\bm{a}_k=\bm{u}_k-\bm{u}^\star$ and $\bm{a}_\infty=\bm{u}_\infty-\bm{u}^\star$, this is
$\bm{u}_k-\bm{u}_\infty=\bm{J}^k(\bm{u}_0-\bm{u}_\infty)$. Expanding in the orthonormal eigenbasis
gives \eqref{eq:mode-convergence}. If
$\langle\bm{u}_0-\bm{u}_\infty,\bm{v}_1\rangle\neq0$, each other eigencomponent decays relative
to the leading component at rate $(\mu_i/\mu_1)^k$, $i\ge2$. The strict spectral gap in
Assumption~\ref{ass:attractor} makes these ratios vanish, proving asymptotically one-dimensional
convergence along $\bm{v}_1$.
\end{proof}

\paragraph{Nonlinear local existence and uniqueness.}
The fixed-point expression in Lemma~\ref{lem:drivenfp} is exact for the linearized loop. For the
nonlinear map, the implicit function theorem gives a nearby fixed point. Write its condition as
$\bm{G}(\bm{u},\bm{f})=\bm{s}(\bm{u})-\bm{u}+\bm{f}=\bm{0}$; it is $C^1$ near
$(\bm{u}^\star,\bm{0})$ (Assumption~\ref{ass:attractor}) with $\bm{G}(\bm{u}^\star,\bm{0})=\bm{0}$,
and its $\bm{u}$-Jacobian there is $\partial_{\bm{u}}\bm{G}=\bm{J}-\bm{I}$, invertible whenever
$\mu_1\neq1$. The contraction $\mu_1<1$ supplies this,
so the theorem gives a neighborhood of $\bm{f}{=}\bm{0}$ on which the
driven fixed point $\bm{u}_\infty(\bm{f})$ exists, is locally unique, and is $C^1$ in $\bm{f}$,
with first-order expansion
$\bm{u}_\infty(\bm{f})=\bm{u}^\star+(\bm{I}-\bm{J})^{-1}\bm{f}+o(\lVert\bm{f}\rVert)$: the linear
formula~\eqref{eq:driven-fp} is its leading term. The theorem upgrades the linear solution to a
locally unique branch of the nonlinear loop; it does not by itself produce $\bm{u}^\star$, whose
existence is Assumption~\ref{ass:attractor}. The empirical check of the near-fixed point is given in
App.~\ref{app:empirical}.

\paragraph{Local frozen-coefficient approximation.}
Lemma~\ref{lem:drivenfp} freezes $\bm{J}$ and $\bm{f}$ (the standard quasi-static
linearization). When $\bm{J}$ and $\bm{f}$ vary sufficiently slowly, the local model describes
tracking of an instantaneous driven fixed point rather than convergence to a stationary one.

\subsection{Difference-of-means estimation of the repetition direction}\label{app:dmean}

The local model identifies an idealized dominant eigenvector $\bm{v}_1$, but this vector is not
available during generation. We therefore estimate an operational direction from observed
self-conditioning feedback: trajectories with more repetition should lie farther along the attractor
than trajectories with less repetition. The following definition makes this contrast precise and
specifies how the resulting direction is used for steering.

\begin{definition}[Difference-of-means direction and steering]\label{def:steer}
Let $a^{(1)}=\langle\bm{a},\bm{v}_1\rangle$ be the signed repetition-axis coordinate,
oriented so $+\bm{v}_1$ points toward the repetition fixed point $\bm{u}^\star$; locally, larger
$a^{(1)}$ is associated with higher REP-4. The trapped and free groups $\mathcal{T},\mathcal{F}$
contain the top and bottom
$q$ fractions of calibration trajectories ranked by REP-4, respectively, where $q\in(0,1/2]$
is the fraction assigned to each tail. They are separated along signed $a^{(1)}$ with
$\bar a^{(1)}_{\mathcal{T}}>\bar a^{(1)}_{\mathcal{F}}$, so the difference-of-means gap
$\bar a^{(1)}_{\mathcal{T}}-\bar a^{(1)}_{\mathcal{F}}\neq0$. The estimator
\begin{equation}
  \bm{d}\;=\;\frac{\bm{m}_{\mathcal{T}}-\bm{m}_{\mathcal{F}}}{\lVert\bm{m}_{\mathcal{T}}-\bm{m}_{\mathcal{F}}\rVert},
  \qquad
  \bm{m}_{\mathcal{T}}=\mathbb{E}[\bm{u}\mid\mathcal{T}],\quad
  \bm{m}_{\mathcal{F}}=\mathbb{E}[\bm{u}\mid\mathcal{F}],
  \label{eq:dmean}
\end{equation}
defines the attractor direction, and steered self-conditioning feeds back
  $\tilde{\bm{x}}_k=\bm{x}_k-\lambda\mathbf{1}_L\bm{d}^{\!\top}$ with strength $\lambda\ge0$.
\end{definition}

\begin{proposition}[Difference of means recovers the attractor direction]\label{prop:dmean}
Split the self-conditioning feedback into the repetition axis and the rest,
$\bm{u}=\bm{u}^\star+a^{(1)}\bm{v}_1+\bm{w}$, where $\bm{w}\perp\bm{v}_1$ is the orthogonal
component. Assume \emph{separability}: the trapped/free split acts only through the on-axis coordinate
$a^{(1)}$, so the orthogonal component $\bm{w}$ has a group-independent mean,
$\E[\bm{w}\mid\mathcal{T}]=\E[\bm{w}\mid\mathcal{F}]$. Then the difference of means is parallel to the
repetition axis, $\bm{d}\parallel\bm{v}_1$.
\end{proposition}
\begin{proof}
Take the group-conditional mean under the model: for $\mathcal{G}\in\{\mathcal{T},\mathcal{F}\}$,
\[
  \bm{m}_{\mathcal{G}}=\E[\bm{u}\mid\mathcal{G}]
  =\bm{u}^\star+\bar a^{(1)}_{\mathcal{G}}\,\bm{v}_1+\E[\bm{w}\mid\mathcal{G}],
  \qquad \bar a^{(1)}_{\mathcal{G}}:=\E[a^{(1)}\mid\mathcal{G}],
\]
the group-mean on-axis coordinate. Subtracting the two groups,
\begin{equation}
  \bm{m}_{\mathcal{T}}-\bm{m}_{\mathcal{F}}
  =\underbrace{(\bm{u}^\star-\bm{u}^\star)}_{=\,\bm{0}}
  +\bigl(\bar a^{(1)}_{\mathcal{T}}-\bar a^{(1)}_{\mathcal{F}}\bigr)\bm{v}_1
  +\underbrace{\bigl(\E[\bm{w}\mid\mathcal{T}]-\E[\bm{w}\mid\mathcal{F}]\bigr)}_{=\,\bm{0}} :
  \label{eq:dmeanderiv}
\end{equation}
the fixed point $\bm{u}^\star$ is shared by both groups and cancels, and the orthogonal term
$\E[\bm{w}\mid\mathcal{T}]-\E[\bm{w}\mid\mathcal{F}]$ vanishes by the modeling assumption. Only the
on-axis term survives, so
$\bm{m}_{\mathcal{T}}-\bm{m}_{\mathcal{F}}=(\bar a^{(1)}_{\mathcal{T}}-\bar a^{(1)}_{\mathcal{F}})\bm{v}_1\parallel\bm{v}_1$,
with scalar $\bar a^{(1)}_{\mathcal{T}}-\bar a^{(1)}_{\mathcal{F}}\neq0$ since the trapped group
sits deeper along $\bm{v}_1$ (Definition~\ref{def:steer}). Normalizing, $\bm{d}\parallel\bm{v}_1$
(up to sign).
\end{proof}

On the trained model, the estimated direction only partially aligns with the leading mode
(Tab.~\ref{tab:jac}).

\subsection{Empirical validation of the attractor model}\label{app:empirical}

We now test the assumptions and predictions of the local attractor model on the trained loop. The
validation covers three questions: whether the feedback has a locally contracting fixed point,
whether one mode becomes dominant as the repetition basin forms, and whether the observed drift
direction differs from the self-conditioning response.
For the fixed-point and Jacobian diagnostics, trajectory base points are captured with the SDE
sampler, while the frozen map and Jacobian are evaluated using the ODE update; the
drift-direction comparison below uses separate deterministic ODE runs.

\paragraph{The frozen pooled loop converges to a fixed point.}
At fraction $0.85$, we capture $100$ SDE trajectories per seed. For each sample, we hold its latent
state and denoising interval fixed, initialize the pooled iterate from its captured pooled feedback,
and apply the deterministic pooled ODE map for $200$ iterations. Across five seeds, the mean relative residual
$\lVert\bm{g}(\bm{u})-\bm{u}\rVert/\lVert\bm{u}\rVert$ decreases from $0.116$ to approximately
$9\times10^{-5}$ (Fig.~\ref{fig:pooled-fixedpoint}a), supporting numerical
convergence of this reduced feedback map.

\paragraph{A dominant contracting mode emerges as the basin forms.}
To determine whether the local contraction is concentrated in a single mode, at each trajectory
fraction we capture a batch of $100$ SDE trajectories and estimate the batch-averaged
Jacobian-vector product of the deterministic ODE update at their captured position-wise feedback
by central finite differences, perturbing uniformly across positions. We use $30$ matrix-free
power iterations for the leading mode and $30$ deflated iterations for the second
(Alg.~\ref{alg:vone}). Table~\ref{tab:jac} reports the mean REP-4 of the $100$ intermediate
predictions at each fraction, alongside the corresponding Jacobian diagnostics. The five-seed
Jacobian-gain and response-share results are shown in Fig.~\ref{fig:pooled-fixedpoint}b,c. We repeat
the measurement at trajectory fractions from $0.4$ through $0.9$.
The spectral gap $\mu_1/\mu_2$ increases from $1.00$ at fraction $0.4$ to $3.04$ at $0.8$, before
slightly decreasing to $2.95$ at $0.9$. This indicates that one mode becomes increasingly separated
as the trajectory enters the repetition regime (Tab.~\ref{tab:jac}). Across five calibration seeds,
trajectory-based probes show leading gains below one and a larger leading-mode response share in
the later repetition regime (Fig.~\ref{fig:pooled-fixedpoint}b,c).
These results support the emergence of a dominant contracting mode as the repetition basin forms.

\begin{table}[t]
\centering
\small
\caption{\textbf{A dominant mode emerges as repetition settles into an attractor.}
  Jacobian diagnostics across denoising trajectory fractions.}
\label{tab:jac}
\setlength{\tabcolsep}{1pt}
\newcolumntype{N}{>{\centering\arraybackslash}m{2.6cm}}
\begin{tabular}{@{}*{5}{N}@{}}
\toprule
    Denoising fraction & REP-4 mean (\%) & Spectral gap $\mu_1/\mu_2$ &
      $\lvert\cos(\bm{v}_1,\bm{d})\rvert$ & Amp. ($\bm{d}$/random) \\
\midrule
    0.40 & 3.2 & 1.00 & 0.51 & 0.87 \\
    0.50 & 7.5 & 1.06 & 0.45 & 1.23 \\
    0.60 & 10.1 & 1.67 & 0.56 & 1.78 \\
    0.70 & 11.3 & 2.41 & 0.56 & 2.20 \\
    0.80 & 11.6 & 3.04 & 0.56 & 2.06 \\
    0.90 & 11.7 & 2.95 & 0.54 & 2.16 \\
\bottomrule
\end{tabular}
\end{table}

\paragraph{The observed drift direction is distinct from the feedback mode.}
Across three deterministic ODE seeds, the leading principal direction of the smoothed per-step
drift over the final third of denoising, estimated from $200$ trajectories per seed, has mean
absolute cosine $0.15$ with the leading Jacobian mode, estimated from $30$ trajectories per seed,
and $0.09$ with the ACE direction. This empirical drift direction summarizes observed updates; it is
not a direct estimate of the theoretical drive $\bm{f}_k$. At fixed latent state and denoising time,
finite differencing the local map cancels its input-independent drive term and isolates the response
to the fed-back clean prediction. Thus,
ACE directly modifies the self-conditioning input; this local analysis does not imply that the full
denoising trajectory is unchanged.
     % formal analysis
\section{Training-Time Mitigation of the Repetition Attractor}
\label{app:training}

The inference-time intervention raises a complementary training-time question: can the model be
trained to avoid producing self-conditioning feedback toward the repetition attractor? We
continue-train ELF-B from the released checkpoint on OpenWebText using AdamW, with weight decay
$0$, effective batch size $512$, learning rate $10^{-4}$, and $128$ optimizer steps. The
anti-attractor run uses four $32$-step blocks: it re-estimates $\bm d$ from $N{=}4096$ Full-SC
calibration samples at $q{=}1/4$ between blocks, and initializes each block from the preceding
EMA checkpoint. The first block uses a $50$-step warmup; subsequent blocks do not repeat it. The
penalty weight is $w{=}8$, while the continued-training control uses $w{=}0$. We evaluate the EMA
weights by plain unsteered sampling on a held-out pool of $1{,}024$ samples, matched to the Full-SC
pool from the main-results run (Table~\ref{tab:main}). The EMA decay is reduced from
the released value $0.9999$ to $0.999$ so that the moving average can track this short run
(Table~\ref{tab:training}).

\suppressfloats[t]
\begin{table}[t]
  \centering\footnotesize
  \caption{\textbf{Penalizing the estimated REP direction during training reduces repetition.}
    The released ELF-B checkpoint is compared with matched continued-training variants with and
    without the anti-attractor penalty on repetition.}
  \label{tab:training}
  \setlength{\tabcolsep}{2.2pt}
  \renewcommand{\arraystretch}{1.08}
  \begin{tabularx}{\textwidth}{@{}>{\raggedright\arraybackslash}p{2.25cm}*{8}{>{\centering\arraybackslash}X}@{}}
    \toprule
    Variant & \metrichead{REP-4}{mean $\downarrow$} & \metrichead{Rep.}{share $\downarrow$} &
      \metrichead{Grammar}{$\uparrow$} & \metrichead{Fluency}{$\uparrow$} &
      \metrichead{dist-$1$}{$\uparrow$} & \metrichead{dist-$3$}{$\uparrow$} &
      \metrichead{s-BLEU}{$\downarrow$} & \metrichead{MAUVE}{$\uparrow$} \\
    \midrule
    Released Full-SC & $6.85$ & $89.36$ & $80.35$ & $1.04$ & $43.46$ & $88.67$ & $29.16$ & $22.50$ \\
    Training control & $6.74$ & $86.91$ & $79.87$ & $1.03$ & $43.48$ & $88.81$ & $28.98$ & $23.65$ \\
    \textbf{Anti-attractor} & $5.42$ & $84.77$ & $80.35$ & $1.02$ & $44.41$ & $90.40$ & $29.59$ & $23.15$ \\
    \bottomrule
  \end{tabularx}
  \renewcommand{\arraystretch}{1}
\end{table}

Rather than modifying the feedback at inference, the anti-attractor variant trains the model not to
produce clean-latent predictions along the repetition direction. Let
$\bm{X}_\theta(\bm{z}_k,t_k)\in\mathbb{R}^{L\times e}$ denote the model's clean-latent prediction
at noise level $t_k$, and let
$\bm{u}_\theta=\tfrac{1}{L}\sum_{l=1}^{L}\bm{X}_\theta[l]\in\mathbb{R}^{e}$ be its
position-pooled prediction, corresponding to the pooled self-conditioning feedback $\bm{u}$ in
\S\ref{sec:theory} and \eqref{eq:scmap}. Using the frozen difference-of-means direction $\bm{d}$
from \eqref{eq:dmean}, we add the following regularizer to the standard flow-matching loss:
\begin{equation}
  \mathcal{L}_{\mathrm{attr}}(\theta)\;=\;
  \mathbb{E}_{\bm{z}_0,\,k}\!\left[\,\operatorname{ReLU}\!\big(\langle\bm{u}_\theta,\,\bm{d}\rangle\big)^{2}\,\right],
  \label{eq:attr}
\end{equation}
and we continue-train against
\begin{equation}
    \mathcal{L}(\theta)\;=\;\mathcal{L}_{\mathrm{base}}(\theta)\;+\;w\,\mathcal{L}_{\mathrm{attr}}(\theta),
    \qquad w\ge 0,
    \label{eq:trainobj}
\end{equation}
where $\mathcal{L}_{\mathrm{base}}$ is the standard flow-matching loss. The one-sided
$\operatorname{ReLU}$ targets only the $+\bm{d}$ half-line associated with the repetitive basin
(Prop.~\ref{prop:dmean}), leaving the $-\bm{d}$ side and all components orthogonal to $\bm{d}$,
denoted by $\bm{w}\perp\bm{d}$, unchanged. Because the regularizer acts on $\bm{X}_\theta$ itself,
the resulting model requires no test-time intervention.

At 128 steps, the anti-attractor variant lowers both REP-4 mean and repetition share relative to
the released Full-SC checkpoint and continued-training control (Table~\ref{tab:training}). The
continued-training control changes REP-4 only marginally relative to Full-SC. This result is
consistent with $\bm{d}$ reflecting a trainable property of the self-conditioning dynamics, rather
than merely an inference-time handle.
   % building the fix into training
\section{Alternative Self-Conditioning Interventions for Repetition}
\label{app:softsc}

We detail the inference-time alternatives evaluated in \S\ref{sec:alternatives}. They operate without
retraining by modifying the clean-embedding prediction before it re-enters the self-conditioning loop.
For ELF-B, under the common $64$-step SDE setting ($\sigma{=}2.0$, $\gamma{=}1.0$, and SC-CFG $w{=}3$),
we generate $1{,}024$ samples directly for each setting using the same held-out generation seed,
without rejection sampling, and report REP-4, repetition share, Grammar, distinct-$1/3$,
self-BLEU, and MAUVE. The Full-SC sampler is retained as a
reference. The compact repetition--quality comparison is shown in Fig.~\ref{fig:alternatives-tradeoff},
and the complete grid over tested alternatives and parameter settings is given in
Tab.~\ref{tab:benchmark-full}.

The alternatives modify the feedback in different ways. \emph{soft-SC mag} feeds back $\alpha$ times
the clean-embedding prediction at every denoising step. Here $\alpha{=}1$ recovers Full-SC, while
$\alpha{=}0$ gives SC-off by disabling self-conditioning feedback during inference; the full $\alpha$ sweep is in
Figure~\ref{fig:sc_strength}. \emph{SC decay} multiplies the feedback by a coefficient that decreases linearly from $1$ at the
first denoising step to a floor $\alpha$ at the last. \emph{SC cutoff} feeds back $\bm{0}$ for all samples
after a fraction $\tau$ of the denoising steps has elapsed. Finally, \emph{SC early-reset}
performs one additional decoder pass on the current clean-embedding prediction after a fraction $\tau$
of the denoising steps. It computes the fraction of token-ID 4-gram occurrences that repeat an earlier
one in the intermediate decoded sequence and, for samples above threshold $\theta$, feeds back
$\bm{0}$ for the remaining steps; other samples retain their original feedback.

ACE instead subtracts $\lambda\bm{d}$ from the clean-embedding prediction
(\S\ref{sec:steer}), while retaining the remaining feedback. The alternatives occupy different
repetition--quality trade-offs: broad feedback suppression can reduce repetition sharply but may lower
Grammar, while less disruptive interventions can retain more repetition. ACE is a simple,
theory-motivated option that targets the estimated repetition direction without adding an auxiliary
decoder pass; it is not uniformly best on every reported metric.

\begin{table*}[t]
  \centering
  \footnotesize
  \caption{\textbf{Self-conditioning interventions trade repetition against text quality.}
    Full-SC, ACE across steering strengths, and the evaluated alternatives are compared across repetition and quality metrics.}
  \label{tab:benchmark-full}
  \setlength{\tabcolsep}{2.0pt}
  \renewcommand{\arraystretch}{1.1}
  \begin{tabularx}{\textwidth}{@{}ll*{8}{>{\centering\arraybackslash}X}@{}}
    \toprule
    \raisebox{\dimexpr(\ht\strutbox+\dp\strutbox+3pt)/2\relax}[0pt][0pt]{Method} &
      \raisebox{\dimexpr(\ht\strutbox+\dp\strutbox+3pt)/2\relax}[0pt][0pt]{Setting} &
      \metrichead{REP-4}{mean $\downarrow$} & \metrichead{Rep.}{share $\downarrow$} &
      \metrichead{Grammar}{$\uparrow$} & \metrichead{Fluency}{$\uparrow$} &
      \metrichead{dist-$1$}{$\uparrow$} & \metrichead{dist-$3$}{$\uparrow$} &
      \metrichead{s-BLEU}{$\downarrow$} & \metrichead{MAUVE}{$\uparrow$} \\
    \midrule
    \mbox{Full-SC} & -- & $7.22$ & $88.67$ & $79.81$ & $1.04$ & $43.26$ & $88.29$ & $29.24$ & $22.21$ \\
    \mbox{SC-off} & -- & $0.09$ & $0.00$ & $51.99$ & $\fluAltSCReset$ & $56.58$ & $99.41$ & $23.02$ & $8.17$ \\
    \midrule
    \multirow{5}{*}{\shortstack[l]{\mbox{ACE}\\[-1pt]$\lambda$}}
      & \mbox{$.25$} & $5.64$ & $82.42$ & $80.58$ & $\fluAceTailQuarterLambdaQuarter$ & $45.28$ & $90.40$ & $29.20$ & $23.39$ \\
      & \mbox{$.50$} & $4.52$ & $75.59$ & $80.82$ & $\fluAceTailQuarterLambdaHalf$ & $46.76$ & $91.93$ & $29.77$ & $19.08$ \\
      & \mbox{$1$} & $3.11$ & $55.37$ & $81.17$ & $\fluAceTailQuarterLambdaOne$ & $48.59$ & $93.89$ & $31.27$ & $11.72$ \\
      & \mbox{$2$} & $1.85$ & $32.62$ & $82.68$ & $\fluAceTailQuarterLambdaTwo$ & $49.79$ & $95.61$ & $35.88$ & $4.89$ \\
      & \mbox{$4$} & $1.58$ & $23.54$ & $83.27$ & $\fluAceTailQuarterLambdaFour$ & $47.48$ & $95.39$ & $44.17$ & $1.02$ \\
    \midrule
    \multirow{4}{*}{\shortstack[l]{\mbox{soft-SC mag}\\[-1pt]$\alpha$}}
      & \mbox{$.2$} & $0.42$ & $2.15$ & $59.68$ & $\fluAltMagA$ & $52.15$ & $98.29$ & $25.84$ & $17.43$ \\
      & \mbox{$.4$} & $1.55$ & $21.58$ & $67.48$ & $\fluAltMagB$ & $47.72$ & $95.87$ & $28.43$ & $23.78$ \\
      & \mbox{$.6$} & $4.62$ & $69.14$ & $72.98$ & $\fluAltMagC$ & $43.56$ & $91.25$ & $29.30$ & $26.48$ \\
      & \mbox{$.8$} & $6.76$ & $86.33$ & $77.12$ & $\fluAltMagD$ & $42.59$ & $88.63$ & $29.33$ & $26.16$ \\
    \midrule
    \multirow{4}{*}{\shortstack[l]{\mbox{SC decay}\\[-1pt]$\alpha$}}
      & \mbox{$0$} & $2.72$ & $39.36$ & $67.09$ & $\fluAltDecayA$ & $47.44$ & $94.31$ & $27.80$ & $23.42$ \\
      & \mbox{$.25$} & $5.16$ & $71.00$ & $71.85$ & $\fluAltDecayB$ & $43.63$ & $90.67$ & $29.01$ & $24.14$ \\
      & \mbox{$.5$} & $6.97$ & $84.86$ & $75.20$ & $\fluAltDecayC$ & $42.03$ & $88.33$ & $29.48$ & $26.26$ \\
      & \mbox{$.75$} & $7.47$ & $87.89$ & $77.85$ & $\fluAltDecayD$ & $42.17$ & $87.81$ & $29.36$ & $26.96$ \\
    \midrule
    \multirow{3}{*}{\shortstack[l]{\mbox{SC cutoff}\\[-1pt]$\tau$}}
      & \mbox{$.25$} & $0.09$ & $0.10$ & $51.83$ & $\fluAltCutoffA$ & $56.39$ & $99.42$ & $23.06$ & $6.51$ \\
      & \mbox{$.50$} & $0.66$ & $7.32$ & $56.29$ & $\fluAltCutoffB$ & $53.16$ & $98.00$ & $25.02$ & $17.38$ \\
      & \mbox{$.75$} & $4.89$ & $73.44$ & $72.75$ & $\fluAltCutoffC$ & $45.90$ & $91.40$ & $28.98$ & $27.72$ \\
    \midrule
    \multirow{5}{*}{\shortstack[l]{\mbox{SC early-reset}\\[-1pt]$(\tau,\theta)$}}
      & \mbox{$(.25,.05)$} & $2.52$ & $31.54$ & $61.72$ & $\fluAltEarlyA$ & $51.81$ & $95.64$ & $24.36$ & $15.16$ \\
      & \mbox{$(.50,.02)$} & $0.97$ & $12.99$ & $58.80$ & $\fluAltEarlyB$ & $52.37$ & $97.46$ & $25.39$ & $18.96$ \\
      & \mbox{$(.50,.05)$} & $2.27$ & $39.75$ & $67.23$ & $\fluAltEarlyC$ & $49.40$ & $95.24$ & $26.93$ & $24.83$ \\
      & \mbox{$(.50,.10)$} & $4.70$ & $74.02$ & $75.95$ & $\fluAltEarlyD$ & $45.71$ & $91.59$ & $28.68$ & $23.64$ \\
      & \mbox{$(.75,.05)$} & $4.94$ & $75.29$ & $73.49$ & $\fluAltEarlyE$ & $45.73$ & $91.29$ & $29.05$ & $26.60$ \\
    \bottomrule
  \end{tabularx}
  \renewcommand{\arraystretch}{1}
\end{table*}

   % alternative self-conditioning interventions
\section{Steering Strength and Direction Analysis}
\label{app:analysis}

This appendix supplements \S\ref{sec:analysis} with analyses of ACE's direction and steering
strength. Algorithm~\ref{alg:steer} summarizes the two-stage procedure: estimating the repetition
direction from calibration trajectories and applying it during generation. We first justify the
calibration and steering settings used in the main results, then compare alternative directions and
controls, and finally examine the local dynamics underlying the repetition attractor.

\Needspace{23\baselineskip}
\begin{algorithm}[H]
\caption{ACE: estimate and subtract the repetition direction}
\label{alg:steer}
\begin{algorithmic}[1]
\Require $N$ Full-SC trajectories summarized by $\{(\bm{s}_n,r_n)\}_{n=1}^{N}$,
  where $\bm{s}_n$ is mean self-conditioning feedback and $r_n$ is sample REP-4;
  tail fraction $q\in(0,1/2]$;
  one-step denoising transition function $\textsc{Step}$, denoising schedule $\{t_k\}_{k=1}^{T}$,
  initial noise scale $\sigma$, steering strength $\lambda$, and token decoder $\textsc{Dec}$
\Ensure repetition direction $\bm{d}$ and one steered sample
\Statex \textbf{Stage 1: Estimate the REP direction from trapped and free trajectories}
\State $\mathcal{T}\gets\{n:r_n\text{ is in the top }q\text{ fraction by REP-4}\}$
\State $\mathcal{F}\gets\{n:r_n\text{ is in the bottom }q\text{ fraction by REP-4}\}$
\State $\bm{d}\gets\operatorname{normalize}\!\left(
  \frac{1}{|\mathcal{T}|}\sum_{n\in\mathcal{T}}\bm{s}_n-
  \frac{1}{|\mathcal{F}|}\sum_{n\in\mathcal{F}}\bm{s}_n\right)$
\Statex \textbf{Stage 2: Steer denoising away from repetition}
\State $\bm{z}_0\sim\mathcal{N}(0,\sigma^2 I),\quad \tilde{\bm{x}}_0\gets\bm{0}$
\For{$k=1,\dots,T$}
  \State $(\bm{z}_k,\bm{x}_k)\gets
    \textsc{Step}(\bm{z}_{k-1},t_k,\tilde{\bm{x}}_{k-1})$
  \State $\tilde{\bm{x}}_k\gets\bm{x}_k-\lambda\mathbf{1}_L\bm{d}^{\!\top}$
    \Comment{broadcast over positions}
\EndFor
\State \Return $\textsc{Dec}(\bm{z}_T)$
\end{algorithmic}
\end{algorithm}

\Needspace{10\baselineskip}
\subsection{Selecting ACE's calibration and steering settings}\label{app:directions}

We select the calibration size $N$ from direction stability across sample sizes, the tail fraction
$q$ from cross-seed agreement, and the steering strength $\lambda$ from the held-out
repetition--quality trade-off. As $N$ grows from $256$ to $4096$, the mean adjacent-size absolute
cosine across the plotted curves rises from $0.917$ to $0.995$, indicating that the direction
stabilizes at the largest tested size
(Fig.~\ref{fig:ace-calibration-stability}a). Among $q\in\{1/8,1/4,1/2\}$, $q{=}1/4$ has the
highest cross-seed agreement (Fig.~\ref{fig:ace-calibration-stability}b and
Fig.~\ref{fig:multiseed-stability-quality}a). Panel (c) compares the model-specific directions
estimated at the selected setting, $N{=}4096$ and $q{=}1/4$. Their pairwise absolute cosines range
from $0.86$ to $0.94$, showing substantial alignment and supporting the shared-direction results in
Table~\ref{tab:main}.

\begin{figure}[!htbp]
  \centering
  \includegraphics[width=0.98\textwidth]{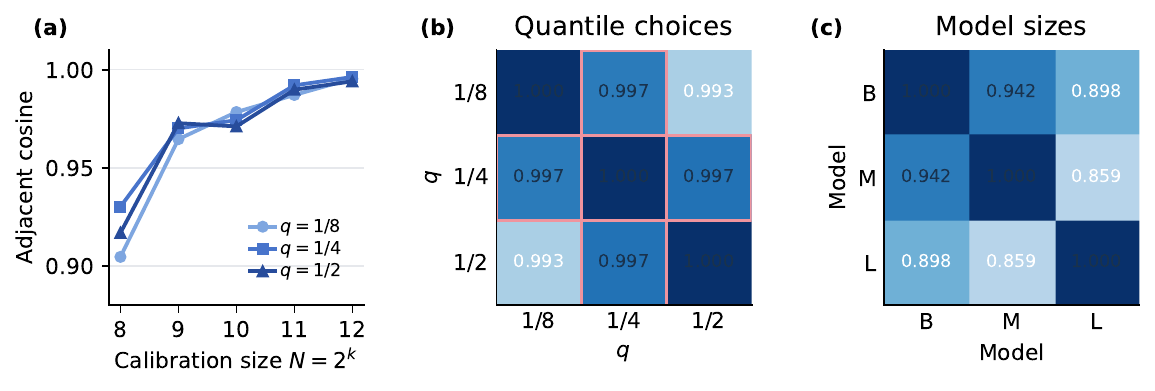}
  \caption{\textbf{The estimated repetition direction is stable across calibration choices and model sizes.}
    (a) Adjacent-size absolute cosine, approaching a stable value at larger $N$; (b) pairwise cosine among the three
    repetition-quantile choices at $N{=}4096$; and (c) pairwise absolute cosine among the model-local
    directions $\bm d_B$, $\bm d_M$, and $\bm d_L$, estimated with $N{=}4096$ and $q{=}1/4$.}
  \label{fig:ace-calibration-stability}
\end{figure}

At $N{=}4096$, the held-out curves show how repetition and quality change with $q$ and steering
strength (Fig.~\ref{fig:ace-q-lambda-response}); the dashed lines mark the Full-SC references.
Increasing $\lambda$ initially reduces repetition while quality remains competitive, whereas
stronger steering exposes a repetition--quality trade-off. We therefore use $\lambda{=}1/2$ as the
quality-focused setting in the main results. Together, these analyses support the settings used in
\S\ref{sec:main_results}: $N{=}4096$, $q{=}1/4$, and $\lambda{=}1/2$.

\begin{figure}[!htbp]
  \centering
  \includegraphics[width=0.98\textwidth]{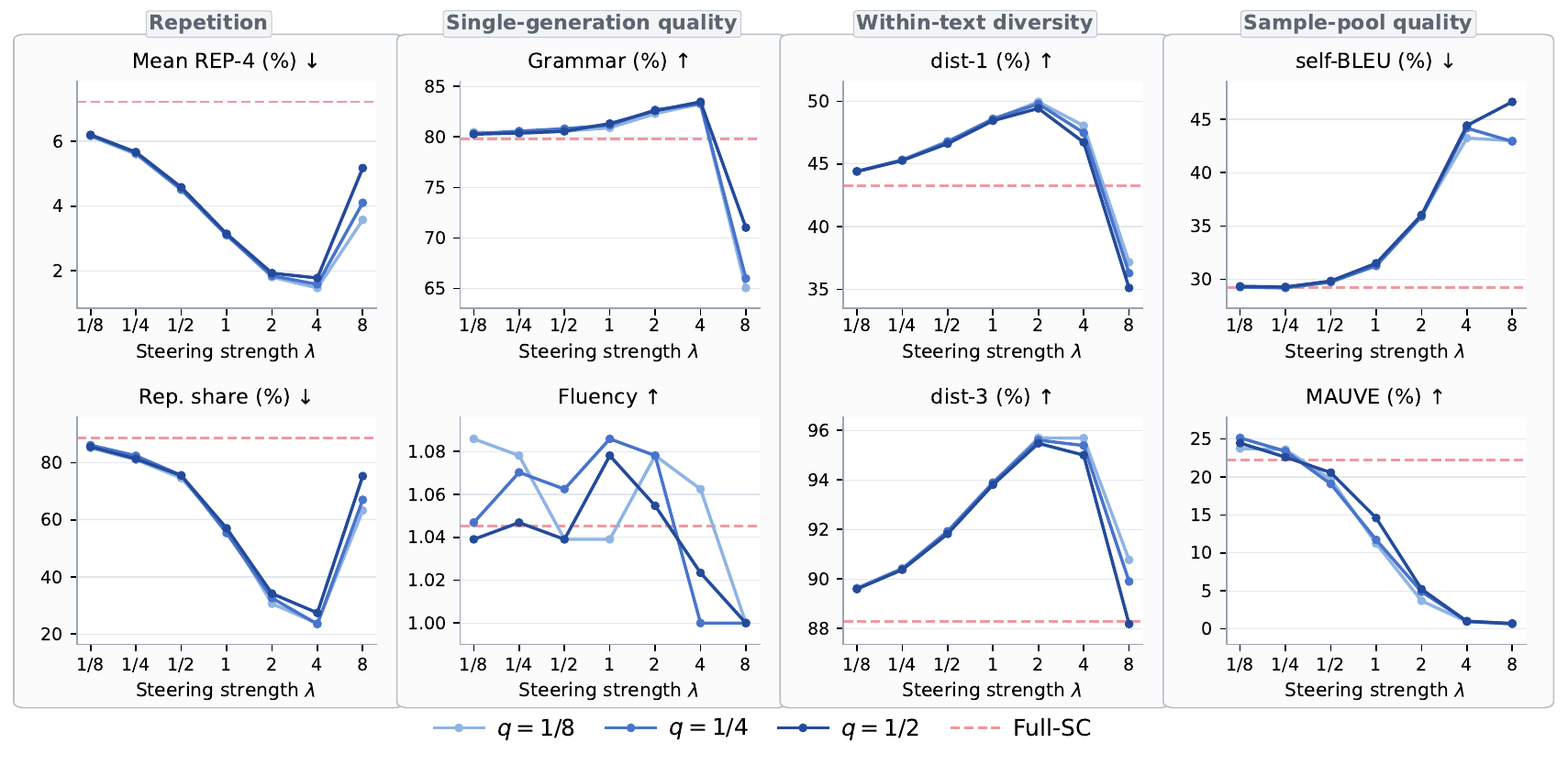}
  \caption{\textbf{ACE exhibits a repetition--quality trade-off across steering strengths.}
    The curves show repetition, single-generation quality, within-text diversity, and sample-pool quality
    as steering strength $\lambda$ varies for the tested calibration tail fractions $q$.}
  \label{fig:ace-q-lambda-response}
\end{figure}

\Needspace{10\baselineskip}
\subsection{Comparing direction estimators and controls}
\label{app:direction-alternatives}

We evaluate how held-out steering depends on the direction estimator. ACE subtracts the mean
self-conditioning feedback of the bottom-$q$ REP-4 tail from that of the top-$q$ tail, omitting the
middle $1-2q$. Logistic and LDA fit scikit-learn's default
\texttt{LogisticRegression} and \texttt{LinearDiscriminantAnalysis} models to the same tails after
feature standardization; their coefficients are mapped back to the original feature space before
normalization. Thus, $q$ denotes the fraction in each tail for ACE, Logistic, and LDA. Top-PC uses
the leading right singular direction of uncentered pooled feedback from the top-$q$ trajectories.
The Jacobian estimator averages local self-conditioning Jacobian-vector products over the same
trajectories and takes the leading mode at $\tau=7/8$. For both estimators, $q=1$ includes all
calibration trajectories; signs are set by the high-minus-low REP-4 feedback contrast, using
$q=1/2$ for this sign convention when $q=1$. A random unit vector and an ACE direction estimated
after permuting REP-4 labels serve as controls.

Figure~\ref{fig:multiseed-stability-quality} shows direction agreement across $q$ in panel (a) and
held-out responses across steering strengths in panel (b). Panel (a) tests $q\in\{1/8,1/4,1/2\}$
for ACE, Logistic, and LDA, plus $q=1$ for Top-PC and the Jacobian estimator. Panel (b) uses each
estimator's $q$ with the highest mean pairwise cosine across five calibration directions: $1/4$ for
ACE, $1/8$ for Logistic, $1/4$ for LDA, and $1$ for Top-PC and the Jacobian estimator. Both panels
use $N{=}4096$ samples and five calibration seeds; panel (b) varies
$\lambda\in\{1/4,1/2,1,2,4\}$, with standard-error bands across the five directions on a shared
$1{,}024$-sample held-out pool. Table~\ref{tab:heldout-controls} gives a matched-strength comparison
under the 64-step SDE setting ($\sigma{=}2.0$, $\gamma{=}1.0$, $w{=}3$), using a separate
$N{=}4096$ calibration set and a common held-out generation seed with $1{,}024$ direct generations
per condition. ACE and the permuted-label control use $q{=}1/4$; the random direction does not use
$q$. All steered rows use $\lambda{=}1/2$, with Full-SC as the reference.

\begingroup
\setlength{\intextsep}{6pt}
\begin{table}[H]
  \centering\footnotesize
  \setlength{\tabcolsep}{2.0pt}
  \caption{\textbf{Repetition-aligned steering reduces repetition relative to unsteered and placebo controls.}
    Held-out comparison of Full-SC, ACE, random-direction, and permuted-label conditions.}
  \label{tab:heldout-controls}
  \renewcommand{\arraystretch}{1.05}
  \begin{tabularx}{\textwidth}{@{}l*{8}{>{\centering\arraybackslash}X}@{}}
    \toprule
    \raisebox{\dimexpr(\ht\strutbox+\dp\strutbox+3pt)/2\relax}[0pt][0pt]{Method} &
      \metrichead{REP-4}{mean $\downarrow$} & \metrichead{Rep.}{share $\downarrow$} &
      \metrichead{Grammar}{$\uparrow$} & \metrichead{Fluency}{$\uparrow$} &
      \metrichead{dist-$1$}{$\uparrow$} & \metrichead{dist-$3$}{$\uparrow$} &
      \metrichead{s-BLEU}{$\downarrow$} & \metrichead{MAUVE}{$\uparrow$} \\
    \midrule
    Full-SC & $7.22$ & $88.67$ & $79.81$ & $1.04$ & $43.26$ & $88.29$ & $29.24$ & $22.21$ \\
    ACE & $4.52$ & $75.59$ & $80.82$ & $\fluHeldACE$ & $46.76$ & $91.93$ & $29.77$ & $19.08$ \\
    Random direction & $8.62$ & $91.89$ & $79.29$ & $\fluHeldRandom$ & $41.36$ & $86.44$ & $30.10$ & $26.36$ \\
    Permuted labels & $7.82$ & $90.14$ & $76.12$ & $\fluHeldPermuted$ & $40.26$ & $87.12$ & $32.14$ & $13.72$ \\
    \bottomrule
  \end{tabularx}
  \renewcommand{\arraystretch}{1}
\end{table}
\endgroup

At this strength, ACE lowers both REP-4 measures relative to Full-SC, whereas neither control does. Across strengths, ACE
and Logistic reduce REP-4 at moderate values while maintaining competitive Grammar and distinct-$3$,
although self-BLEU rises with stronger steering. Top-PC also lowers REP-4 at $\lambda{=}4$, but with
a marked increase in self-BLEU; LDA and the Jacobian mode show less consistent REP-4 reduction.
Together, these results detail the \S\ref{sec:analysis} finding that theory-motivated estimators can
reduce repetition while retaining competitive quality at suitable strengths, whereas theory-agnostic
alternatives do not reliably reproduce this joint behavior. ACE provides a simple, stable, and
effective implementation of this principle.

\begin{figure}[!t]
  \centering
  \includegraphics[width=0.98\textwidth]{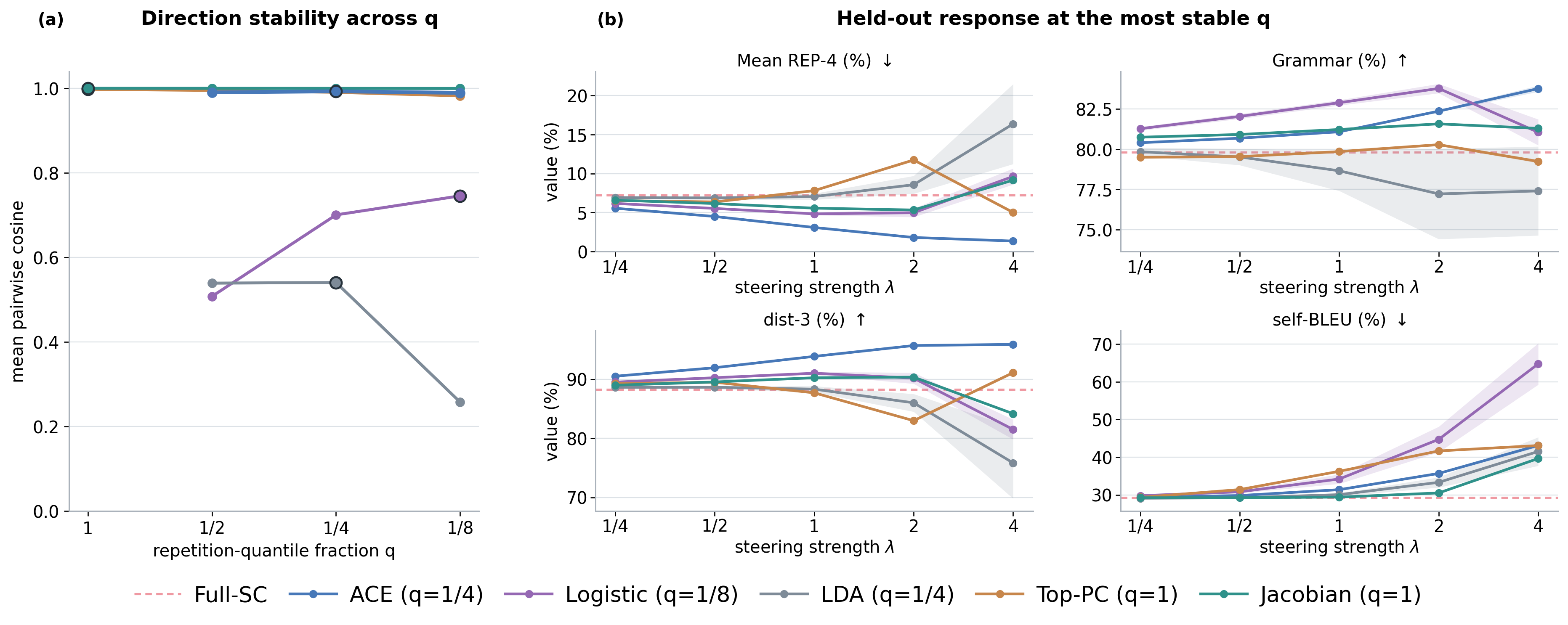}
  \caption{\textbf{Calibration stability and held-out responses distinguish repetition-aligned directions.}
    (a) Mean pairwise cosine across calibration directions as $q$ varies. (b) Held-out repetition
    and quality as steering strength $\lambda$ varies at each estimator's selected $q$.}
  \label{fig:multiseed-stability-quality}
\end{figure}

\subsection{Dynamical validation and interpretation of the repetition direction}
\label{app:direction-dynamics}

The fixed-point and mode-dominance diagnostics are detailed in App.~\ref{app:empirical}; here we
focus on how ACE's estimated direction relates to the local feedback mode. We estimate Jacobian
modes from central finite-difference products of the deterministic ODE transition at SDE-captured
trajectory points, retaining position-wise feedback. We evaluate fractions
$f\in\{0.4,0.5,0.6,0.7,0.8,0.9\}$, average over $100$ trajectories per seed and five calibration
seeds, and use $30$ power iterations for the leading mode and $30$ deflated iterations for the
second (Alg.~\ref{alg:vone}). The compared ACE direction uses $N{=}4096$ and $q{=}1/4$; $q$ is not
a Jacobian-estimator parameter. By contrast, the fixed-point check applies a frozen pooled map,
which broadcasts pooled feedback (App.~\ref{app:empirical}). We orient $\bm{v}_1$ so its feedback
projection increases with repetition. The finite-iteration gains are approximate mode diagnostics;
the orthogonal eigenmode interpretation relies on the symmetric idealization.

As the trajectory enters the repetitive regime, the spectral gap rises from $1.00$ at fraction
$0.40$ to $3.04$ at $0.80$, then is $2.95$ at $0.90$ (Tab.~\ref{tab:jac}; Fig.~\ref{fig:mechanism}b).
Across these fractions, the absolute cosine between $\bm{d}$ and $\bm{v}_1$ is $0.45$--$0.56$,
above the random reference $1/\sqrt{e}=0.044$ but far from identity; the amplification ratio along
$\bm{d}$ relative to random perturbations reaches $2.20$. The held-out curves show less consistent
repetition reduction for the Jacobian estimator than for ACE (Fig.~\ref{fig:multiseed-stability-quality}).
Thus, $\bm{d}$ is partially aligned with, but is not
identical to, the leading local mode. This is consistent with ACE capturing a direction that recurs
in the evolving feedback dynamics, rather than a single static Jacobian eigenvector.

\begin{algorithm}[H]
\caption{Attractor modes $\bm{v}_1,\bm{v}_2$: matrix-free deflated power iteration on the self-conditioning feedback Jacobian}
\label{alg:vone}
\begin{algorithmic}[1]
\Require one-step map $\bm{g}$ at trajectory fraction $f$; power-iteration count $K{=}30$;
  probe scale $\varepsilon$
\Ensure modes $\bm{v}_1,\bm{v}_2$ and gains $\mu_1,\mu_2$ (spectral gap $\mu_1/\mu_2$)
\State iterate the loop $\bm{g}$ forward (baseline run) and capture its self-conditioning feedback $\bm{x}_k$ at step $k{=}\lfloor fT\rfloor$ (near-converged)
\State $\bm{u}^\star\gets \tfrac{1}{L}\sum_{l=1}^{L}\bm{x}_k[l]$
  \Comment{captured from the run, not solved}
\State $\varepsilon\gets 0.01\,\operatorname{std}(\bm{u}^\star)$
\State $\bm{J}\bm{v}:=\big[\bm{g}(\bm{u}^\star{+}\varepsilon\bm{v})-\bm{g}(\bm{u}^\star{-}\varepsilon\bm{v})\big]/(2\varepsilon)$
  \Comment{drive $\bm{f}_k$ cancels}
\State $\bm{v}_1\sim\mathcal{N}(\bm{0},\bm{I})$, normalized
\For{$j=1,\dots,K$}
  \State $\bm{w}\gets\bm{J}\bm{v}_1$;\quad $\mu_1\gets\lVert\bm{w}\rVert$;\quad $\bm{v}_1\gets\bm{w}/\mu_1$
    \Comment{\citep{golub2013matrix}}
\EndFor
\State $\bm{v}_2\sim\mathcal{N}(\bm{0},\bm{I})$, orthogonalized against $\bm{v}_1$ and normalized
\For{$j=1,\dots,K$}
  \State $\bm{w}\gets\bm{J}\bm{v}_2-\langle\bm{J}\bm{v}_2,\bm{v}_1\rangle\bm{v}_1$;\quad
    $\mu_2\gets\lVert\bm{w}\rVert$;\quad $\bm{v}_2\gets\bm{w}/\mu_2$
    \Comment{deflated}
\EndFor
\State \Return $\bm{v}_1,\bm{v}_2,\mu_1,\mu_2$
\end{algorithmic}
\end{algorithm}

To measure how much of the local feedback response is captured by the leading mode, we draw
$P$ random unit perturbations $\{\bm{\delta}_j\}_{j=1}^{P}$ and evaluate the finite-difference
responses $\bm{J}\bm{\delta}_j$. Under the local symmetric approximation, we summarize the
response energy along $\bm{v}_1$ by
\begin{equation}
  \rho_1 =
  \frac{\sum_{j=1}^{P}\langle \bm{J}\bm{\delta}_j,\bm{v}_1\rangle^2}
       {\sum_{j=1}^{P}\lVert \bm{J}\bm{\delta}_j\rVert^2}.
  \label{eq:response-energy}
\end{equation}
A value near one would mean that nearly all local response energy lies along $\bm{v}_1$. The
measured share rises late in denoising but remains below one (Fig.~\ref{fig:pooled-fixedpoint}c),
supporting an increasingly concentrated, rather than exactly rank-one, response. The leading gains
remain below one, and the separate frozen pooled-map iteration converges numerically
(Fig.~\ref{fig:pooled-fixedpoint}a,b; App.~\ref{app:empirical}). Together, the two diagnostics
support contraction in the local maps and increasing concentration of the trajectory-level response
in one mode; they do not establish an exactly one-dimensional system.

\begin{figure}[t]
  \centering
  \includegraphics[width=0.99\linewidth]{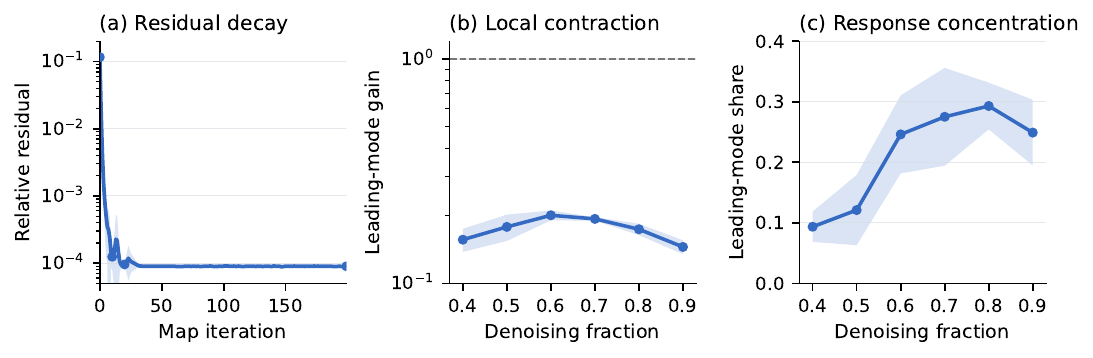}
  \caption{\textbf{The frozen pooled map converges, while the trajectory-level response concentrates in one leading mode.}
    (a) Frozen pooled-map iteration reduces the relative residual. (b) The leading finite-difference
    gain remains below one. (c) The leading-mode response share increases in the later repetition
    regime. Shaded bands show standard deviations across five calibration seeds.}
  \label{fig:pooled-fixedpoint}
\end{figure}

The temporal probes provide a complementary view: across samples, projection onto $\bm{v}_1$
tracks REP-4, while along Full-SC trajectories the projection onto $\bm{d}$ builds as REP-4 rises;
ACE suppresses both the projection and repetition (Fig.~\ref{fig:mechanism}a,c,d). This scalar
projection trend is distinct from the direction of the per-step drift, whose leading principal
direction has low mean absolute cosine similarity with both the Jacobian mode and $\bm{d}$; as
discussed in App.~\ref{app:empirical}, it is not a direct estimate of the theoretical drive.
The association suggests that $\bm{d}$ tracks repetition-related feedback during denoising, though
it does not by itself validate a detector or intervention based on that signal. As a token-level
probe, a linear logit lens through the output unembedding ranks subword pieces such as \textit{tri},
\textit{ction}, \textit{iding}, and \textit{lect} among its top-scoring tokens, while whole words
such as \textit{emotion}, \textit{scholarship}, and \textit{stretched} rank among its lowest. This
pattern is consistent with $\bm{d}$ encoding repetitive sub-lexical content that steering removes.
   % steering strength and direction analysis
\section{Robustness and Generalization}
\label{app:robust}

This appendix supplements \S\ref{sec:gen}. Within ELF, we test transfer of the ELF-B direction across
model sizes and inference configurations, and separately assess robustness to calibration and
generation seeds. For these ELF experiments, the canonical setting uses $64$ SDE denoising steps,
SC-CFG guidance scale $w{=}3$,
noise scale $\sigma{=}2.0$, and churn $\gamma{=}1.0$; ACE uses $\lambda{=}1/2$.
Each ELF robustness-table run contains $1{,}024$ direct generations, and we report repetition and
text-quality metrics for the Full-SC baseline and ACE with the shared direction.
The direction-seed block averages five independently estimated directions on a common held-out pool;
size and robustness rows average five held-out seeds, while step and guidance rows use a single
held-out seed.

\subsection{Robustness across inference configurations}
\label{app:robust-config}

We vary the number of denoising steps, guidance scale, initial noise scale
$\sigma\in\{1.5,2.5\}$, ODE and SDE sampling, SDE churn, and random seed.
Figure~\ref{fig:knobs} displays five denoising-step values,
$16,32,64,128,256$, and five SC-CFG guidance scales, $1,2,3,4,5$; each point reports mean REP-4
from $1{,}024$ samples. The step panel fixes the SC-CFG scale at $3$, while the guidance panel fixes the
number of denoising steps at $64$. Table~\ref{tab:crosssize} reports the corresponding
quality measures for these settings and for the additional noise, churn, and ODE configurations.

\subsection{Transfer across model sizes}
\label{app:robust-size}

The direction estimated on ELF-B also transfers to ELF-M and ELF-L. Table~\ref{tab:crosssize}
reports the shared direction on five held-out seeds and shows the same reduction in repetition while
the quality measures remain competitive. Model-specific direction estimates are reported
separately in the main results table.

\subsection{Transfer to other self-conditioned models}
\label{app:robust-model}

To test cross-model transfer, we select unconditional continuous-diffusion models that use
self-conditioning and have released checkpoints. We exclude models that combine left-to-right
blockwise generation with within-block diffusion, such as AURORA-LM~\citep{liang2026aurora}
and ByteDance's CoLA-DLM~\citep{guo2026continuous}. The resulting test models are
Plaid~\citep{gulrajani2023likelihood} and LangFlow~\citep{chen2026langflow}.

For Plaid, direction calibration uses $1024$ Full-SC samples and $q{=}1/3$; the table reports
$\lambda{=}2$. Plaid uses $4096$ denoising steps and score
temperature $0.9$.
For LangFlow, we use $q{=}1/2$ and its $1024$-step Euler--EDM sampler. Like ELF, ACE exhibits a
repetition--quality trade-off on these external models. Across $\lambda\in\{1/4,1/2,1,2,4\}$,
LangFlow's REP-4 falls as strength increases through $\lambda{=}2$ while single-generation quality
remains near Full-SC; at stronger steering, sample-pool metrics worsen and REP-4 rebounds. This
parallels the strength-dependent pattern discussed for ELF in \S\ref{sec:analysis}.
Table~\ref{tab:crossmodel} reports $\lambda{=}1/2$, matching the main-text reference strength.

\begin{figure}[!htbp]
  \centering
  \includegraphics[width=0.98\textwidth]{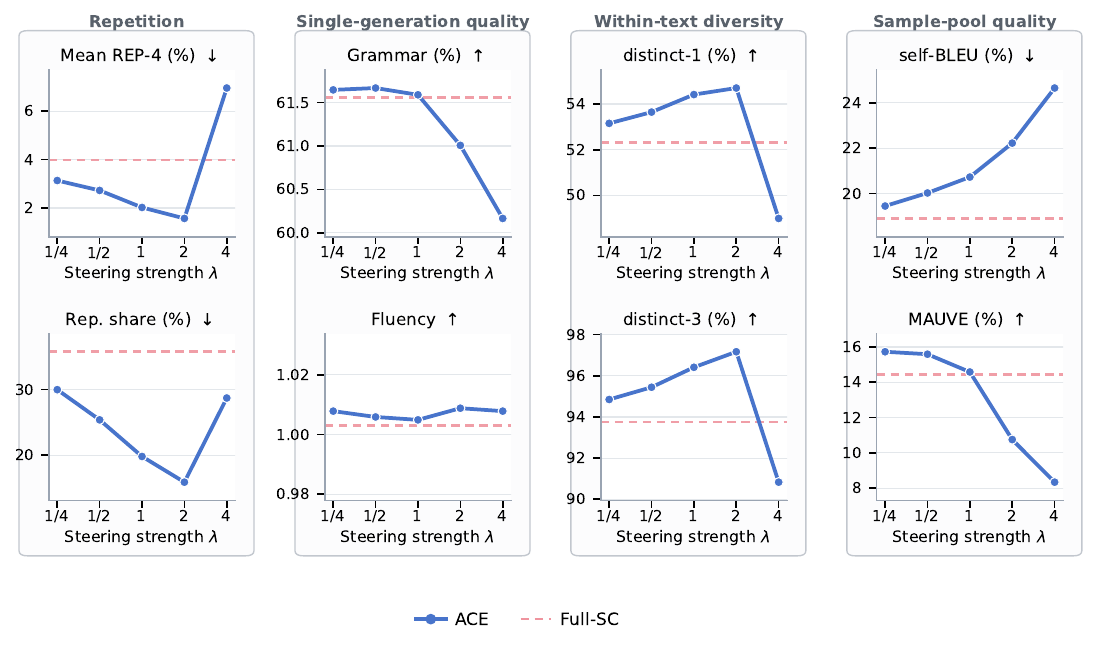}
  \caption{\textbf{LangFlow shows a strength-dependent repetition--quality trade-off.}
    Mean REP-4 and repetition share, Grammar and Fluency, distinct-$1/3$, and self-BLEU and MAUVE
    across ACE steering strengths; dashed lines mark Full-SC.}
  \label{fig:langflow-strength-response}
\end{figure}

Together, these results show that ACE is robust across the tested ELF configurations and model sizes
and also reduces repetition in two external self-conditioned continuous diffusion models using
model-specific directions.

\begin{table*}[t]
  \centering
  \footnotesize
  \caption{\textbf{ACE transfers across model sizes and inference configurations, remains stable
    across independent seeds, and maintains competitive text quality.} Full-SC and ACE are compared across model sizes, inference
    configurations, and independent seeds.}
  \label{tab:crosssize}
  \setlength{\tabcolsep}{2.0pt}
  \renewcommand{\arraystretch}{1.0}
  \newcommand{\dirmeanstd}[2]{#1_{\scalebox{0.70}{$\scriptscriptstyle\pm#2$}}}%
  \begin{tabularx}{\textwidth}{@{}ll*{8}{>{\centering\arraybackslash}X}@{}}
    \toprule
    \raisebox{\dimexpr(\ht\strutbox+\dp\strutbox+3pt)/2\relax}[0pt][0pt]{Setting} &
      \raisebox{\dimexpr(\ht\strutbox+\dp\strutbox+3pt)/2\relax}[0pt][0pt]{Method} &
      \metrichead{REP-4}{mean $\downarrow$} & \metrichead{Rep.}{share $\downarrow$} &
      \metrichead{Grammar}{$\uparrow$} & \metrichead{Fluency}{$\uparrow$} &
      \metrichead{dist-$1$}{$\uparrow$} & \metrichead{dist-$3$}{$\uparrow$} &
      \metrichead{s-BLEU}{$\downarrow$} & \metrichead{MAUVE}{$\uparrow$} \\
    \midrule
    \multicolumn{10}{@{}l}{\emph{Canonical ELF-B configuration (64-step SDE, $\sigma{=}2$, $\gamma{=}1$, $w{=}3$)}} \\
    \multirow{2}{*}{Reference} & Full-SC & $7.28$ & $90.00$ & $80.21$ & $1.04$ & $43.24$ & $88.19$ & $29.03$ & $23.63$ \\
     & ACE & $4.48$ & $74.02$ & $80.57$ & $\fluCrossReferenceACE$ & $46.75$ & $91.93$ & $29.61$ & $19.74$ \\
    \midrule
    \multicolumn{10}{@{}l}{\emph{Direction estimation seeds}} \\
    \multirow{2}{*}{5 $\bm d_B$ estimates} & Full-SC & $7.22$ & $88.67$ & $79.81$ & $1.04$ & $43.26$ & $88.29$ & $29.24$ & $22.21$ \\
     & ACE & $\dirmeanstd{4.50}{0.06}$ & $\dirmeanstd{74.57}{0.60}$ & $\dirmeanstd{80.69}{0.06}$ & $\fluCrossDirectionACE$ & $\dirmeanstd{46.81}{0.05}$ & $\dirmeanstd{91.94}{0.07}$ & $\dirmeanstd{29.84}{0.13}$ & $\dirmeanstd{19.27}{0.73}$ \\
    \midrule
    \multicolumn{10}{@{}l}{\emph{Transfer across model sizes}} \\
    \multirow{2}{*}{ELF-M} & Full-SC & $2.99$ & $52.42$ & $85.97$ & $\fluCrossSizeMFull$ & $47.88$ & $94.04$ & $27.10$ & $25.05$ \\
     & ACE & $1.79$ & $31.84$ & $85.56$ & $\fluCrossSizeMACE$ & $50.82$ & $96.00$ & $27.16$ & $22.56$ \\
    \multirow{2}{*}{ELF-L} & Full-SC & $2.50$ & $43.61$ & $85.83$ & $\fluCrossSizeLFull$ & $50.14$ & $94.94$ & $27.27$ & $23.04$ \\
     & ACE & $1.58$ & $26.60$ & $85.60$ & $\fluCrossSizeLACE$ & $52.44$ & $96.45$ & $27.51$ & $19.81$ \\
    \midrule
    \multicolumn{10}{@{}l}{\emph{Denoising steps}} \\
    \multirow{2}{*}{32 steps} & Full-SC & $2.17$ & $42.58$ & $77.88$ & $1.04$ & $48.03$ & $95.07$ & $28.34$ & $18.73$ \\
     & ACE & $1.33$ & $22.85$ & $77.18$ & $\fluCrossStepsShortACE$ & $51.23$ & $96.69$ & $27.50$ & $15.95$ \\
    \multirow{2}{*}{128 steps} & Full-SC & $15.29$ & $99.02$ & $79.71$ & $1.04$ & $38.29$ & $78.94$ & $28.61$ & $25.51$ \\
     & ACE & $10.06$ & $94.73$ & $80.56$ & $\fluCrossStepsLongACE$ & $42.46$ & $85.08$ & $29.88$ & $23.46$ \\
    \midrule
    \multicolumn{10}{@{}l}{\emph{SC-CFG guidance scale}} \\
    \multirow{2}{*}{$w=2$} & Full-SC & $5.01$ & $76.56$ & $79.20$ & $1.04$ & $44.80$ & $91.07$ & $28.74$ & $22.13$ \\
     & ACE & $2.88$ & $55.66$ & $79.14$ & $\fluCrossGuidanceLowACE$ & $48.53$ & $94.22$ & $28.88$ & $21.10$ \\
    \multirow{2}{*}{$w=4$} & Full-SC & $6.78$ & $89.06$ & $80.81$ & $1.04$ & $43.33$ & $88.64$ & $29.92$ & $23.66$ \\
     & ACE & $4.31$ & $76.56$ & $81.00$ & $\fluCrossGuidanceHighACE$ & $46.80$ & $92.05$ & $30.34$ & $20.93$ \\
    \midrule
    \multicolumn{10}{@{}l}{\emph{Additional inference configurations}} \\
    \multirow{2}{*}{$\sigma=1.5$} & Full-SC & $15.80$ & $99.45$ & $81.96$ & $1.04$ & $34.19$ & $77.25$ & $40.56$ & $6.92$ \\
     & ACE & $9.99$ & $94.75$ & $82.15$ & $\fluCrossNoiseLowACE$ & $39.41$ & $84.43$ & $38.17$ & $8.16$ \\
    \multirow{2}{*}{$\sigma=2.5$} & Full-SC & $8.05$ & $79.20$ & $72.81$ & $1.04$ & $43.95$ & $87.34$ & $22.03$ & $35.07$ \\
     & ACE & $4.36$ & $62.81$ & $75.25$ & $\fluCrossNoiseHighACE$ & $47.64$ & $92.12$ & $24.93$ & $28.96$ \\
    \multirow{2}{*}{$\gamma=0.5$} & Full-SC & $6.34$ & $83.46$ & $77.16$ & $1.04$ & $45.11$ & $89.60$ & $26.34$ & $26.49$ \\
     & ACE & $3.85$ & $65.43$ & $77.75$ & $\fluCrossChurnLowACE$ & $48.23$ & $92.94$ & $27.30$ & $21.83$ \\
    \multirow{2}{*}{$\gamma=1.5$} & Full-SC & $8.57$ & $93.57$ & $82.40$ & $1.04$ & $41.21$ & $86.34$ & $31.88$ & $18.15$ \\
     & ACE & $5.33$ & $81.15$ & $82.49$ & $\fluCrossChurnHighACE$ & $45.17$ & $90.65$ & $31.99$ & $16.15$ \\
    \multirow{2}{*}{ODE sampler} & Full-SC & $5.55$ & $76.58$ & $72.50$ & $1.04$ & $47.01$ & $90.78$ & $23.74$ & $29.46$ \\
     & ACE & $3.34$ & $58.01$ & $73.72$ & $\fluCrossODEACE$ & $49.85$ & $93.75$ & $24.93$ & $25.56$ \\
    \bottomrule
  \end{tabularx}
  \renewcommand{\arraystretch}{1}
\end{table*}

 % robustness and generalization
\section{Qualitative Examples}
\label{app:qualitative}

We provide two paired ELF-B generations to make the repetition defect and the effect of ACE
concrete. Within each pair, Full-SC and ACE use the same inference setting and random seed.
For readability, red boxes highlight all occurrences of one repeated 4-gram in each Full-SC
output. The examples
illustrate the repetition effect; aggregate text quality is evaluated in Table~\ref{tab:main}.

\newcommand{\repbox}[1]{\fcolorbox{red!75!black}{red!8}{\strut #1}}

\subsection{Paired generation 1}
\label{app:qual-pair-1}

\noindent\emph{Full-SC REP-4: $23.3\%$; ACE REP-4:
$1.5\%$. The highlighted 4-gram, ``button on the top'', occurs 8 times in
Full-SC.}

\begin{mdframed}[style=examplebox]
\textbf{Full-SC (baseline)}\par\smallskip
\small Want more the top college football news you've been watching? Make be sure to follow us on Twitter: www.leadyball.com If you're a beginner, today is the perfect time to create a custom logo for your website. Before you start doing anything else, consider also designing your logo for your website. This post will overview the best ways to create a logo, how to design for your website, and picking the best one you have. tips: 1. Design The Plams Creating a custom logo is a bit of a tedious process, but we can usually do everything we can to provide a design so that a logo can be much more visible and better shared with the public. Photos and videos are essential to creating a logo, so it's always important to keep your website's logo visible as well. Remember to make your logo a single photo. Read a tutorial on how to do this below. 2. Create The Logo The first step of creating the logo, click the ``CALL'' \repbox{button on the top} of the window and to select the website logo. Click on ``Define'' to create the new logo. To create the logo, click the ``Distribuates'' \repbox{button on the top} of the window and select to select the website logo. Click on ``Distribuates'' to create the new logo. The next step, select to be your favorite type of website and select to be your favorite type of website. To design the new logo, click the ``Defate'' button on the bottom of the window. After you start creating the logo, you can select the type of websites you would love to use. To create the new logo, click the ``Detribuates'' \repbox{button on the top} of the window and select the links at the fillin of the next window. Click on logo in the step below. 3. Customate The Logo Once you've created your own logo, you can click the Design button above your window. Now, you can design your own logo, and be sure to select the website that you like the most. Here's where you start to create: 4. Design The Logo First step of creating the logo, you first step by selecting the appropriate type of websites you would love to use. Next step is to select the appropriate type of websites you'd love to use. In this video, be sure to design your logo so that it remains visible. The image below will show you everything you really need to do to design the logo. You can design a custom website logo to make sure it remains very visible on your website. The first step of creating the logo, click the ``Distribuates'' button to select to be your favorite type of website. Click the ``Define'' \repbox{button on the top} of the window to select the your favorite type of website and then click on logo to select the website. Click the ``Detribuates'' \repbox{button on the top} of the window Your next step is to select the type of websites you want to use. Click the ``Distribuates'' button to select the type of websites. Click the ``Defate'' \repbox{button on the top} of the window to select your favorite type of website. 4. Customate The Website Logo The next step involves selecting the type of websites you want to use. To start creating the logo, click the ``Conturn'' button on the bottom of the window to select the favorite website logo you like. Your next step is to select the type of websites you would love to use on your website. After you start creating the logo, you can select the appropriate websites you would love to use and click the Printdown button in the next window. 5. Design Your Own Delete The Website Logo The first step of creating the logo, you can custom a new logo. Click the ``Delete'' \repbox{button on the top} of the window to select your favorite website logo and the ``Printdown'' button to select to remove the website that you like. After you create the logo, you can select a logo that you like. In order to select your favorite website logo to change, click the Printdown button in the next window. To select your website logo to change into a new design, click the ``Delete'' \repbox{button on the top} of the window to select your favorite website logo. Click on logo in step below. 6. Customate The Website Logo The next step involves selecting the favorite website you would love to use. Click the ``Tain'' button to select your favorite website logo. Click the ``Define'' button to select your favorite website logo. To select the logo you like, click the `` Like'' button above your window to learn about your favorite website logo. Click the ``Re Remove'' button to select the most popular website you love. Click the ``Delete'' button to select your favorite website logo. Click ``Conturn'' button to select the most popular website you love. Select a website logo you like.
\end{mdframed}

\begin{mdframed}[style=examplebox]
\textbf{ACE}\par\smallskip
\small LONDIN (Reuters) - Thousands of people rallied in a building in Bench City, the financial sector capital of the Western capital, on Tuesday for a march on White House Day, as many of their opposition supporters battle over their government. Supporters of Syrian President Bashar al-Assad rally on the streets of Bench City during a demonstration against President Assad in the West African capital, July 19, 2015. REUTERS/Carel Jadder Al-Assad's ouster of Syrian president on Tuesday sparked a smearing of opposition by U.S. forces in Syria over the Turkish government, something that has ringed in the country for more than two decades. Opponent had said that Assad's ouster hadposed a direct threat to the U.S. security system. ``The elections are taken very seriously.'' Interior Minister John Haden said a defender could not provide further details about the event. The march also became the largest opposition march in Damascus to be carried out by more than 700 people. Nearly 32,000 rallied outside the Sidran conference centre, which houses more than 1,100 people. Trade unions, pro-government opposition groups, workers' rights, and business and television networks were also participants. Nearly 2,000 people had attended the event of the opposition march. Among the participants, Rafel Khajdi, a member of the National Security Council, a right-wing organization in Washington and Moscow, was among those who attended the march. Seri Belocki, a former U.S. diplomat who backed Russia's Republican presidential nominee Donald Trump, was among those who attended the march. Alandzi Dakhev, a former member of the security council under former President George W. Bush also attended the march. Idel Sidalah, a 22-year-old online journalist with a PhD in politics, had also participated in the march; he was not the only people he said attending the march. The event was the latest in a string of violent and intense marchs in which more than 2,000 people, members of Congress and government officials from over 100 countries joined in battles for change over the ruling Assad government, which is still rifling fighting. Business leaders rallied back into Tuesday afternoon to voice their words outside the headquarters of an oppositionist group run by the Alliance for Human Rights in Benghazi, a south-west part of the country. T clashes were intensed after a policeman was shot dead in a airstrike on government forces in Bench City on Tuesday morning. Simael Nashloud, an Cairo-based online journalist, said: ``It's scary.'' Almost everyone among thousands of opposition and opposition protesters gathered in front of a building in business centre in Bench City on Tuesday. They chantted rings of shakers, machine-guns, and drones, latched off gas and gas tunnels, cast off rail tunnels, and threw mini-packs of iron. ``Thirety people were injured and many others were still breathing,'' said Deni Mosar, a Sunni opposition journalist, who lives on Cairo's 3k border and attended the march. ``This was a very dangerous march.'' Many Sunnis have flew to Syria in opposition to opposition legislation introduced in July calling for the Syrian president Bashar al-Assad's resignation and a referendum to remove the Syrian government. Hundreds of thousands of opposition protesters have turned up at the march on social media, according to the movement. Yar al-Malbar, a 27-year-old who works for a low-level security force who was recently killed in talks with government forces, also participated in the march. Executivees of the Armed Syria Security Service, whose intelligence services work against Assad flew from Syria, Iraq and northeastern Syria. ``If we were just going to get away, I think we would migrate to a very safe place,'' said Mohammed al-Jarian, a 23-year-old Sunni journalist from Egypt who was also participating in the march. ``Our mission is for those who are victims, just not innocents.'' Hundreds of registered opposition protesters held a chant of support for Assad's Syrian government, many of them chanting ``Thanks to God''. Mohammed al-Arraq Mahyef, a 27-year-old journalist from Benghazi, was among the first to attend the event of the opposition march. Ali Aster, a 27-year-old journalist from Egypt, also attended a march in Egypt.
\end{mdframed}

\subsection{Paired generation 2}
\label{app:qual-pair-2}

\noindent\emph{Full-SC REP-4: $13.3\%$; ACE REP-4:
$0.8\%$. The highlighted 4-gram, ``sit in a row'', occurs 10 times in
Full-SC.}

\begin{mdframed}[style=examplebox]
\textbf{Full-SC (baseline)}\par\smallskip
\small By the time I was about 10, I was about to get into rain and I didn't know how I would get better. When I was about 10, I was angry. I wanted to stand up to what was going on. I knew how to stay down. Whenever the rain was over, I had to clean myself up. I wanted to \repbox{sit in a row}, watch new movies and watch television. Once my mom was ill, she had to place her finger on the back of her head. I wanted to read the book, watch new movies, read history and watch television. I wanted to write the book about Professor Robert Goldberg, the Editor of the Guardian. I would ask my mom if she could get a job back. I would ask for my friend. I wanted to \repbox{sit in a row}. I always thought I had to ask myself, but I couldn't sit if I wanted to write something about what I was talking about. I wanted to write the book about the Guardian. I had to \repbox{sit in a row}. I had to \repbox{sit in a row}. I wanted to be a journalist. I wanted to be the Editor of the Guardian. I had to \repbox{sit in a row}. I had to have noses and throats and all kinds of bruises. I had to control myself out. Whenever the rain was over, I had to control my mind, my body, my heart. I had some of the worst experiences I ever had to have. I had to \repbox{sit in a row}. I would have to be sitting on a friend's table or on my bed. It would blow me. I wanted to be a journalist but I never wanted to tell new stories. It was scary. I had no idea of where I was. I felt it was my way of life. I wasn't able to feel I could get a job from family. I would check my kids out. I would feel tired. I would re-entry myself. Then I had to \repbox{sit in a row}. I would go home, I would open doors, I would pick my job. I had to have my head hammed against the wall and my bed. I didn't have to, I had to \repbox{sit in a row}. I thought I could do anything I possibly could. I was raped, beaten, savaged, thrown hard by the legs I had on. Then I walked away. I thought I should go to school with friends. I thought a man should just get out on a wall and let rain come out. I thought I should get out on a wall and let no rain come out. I wanted to \repbox{sit in a row} and let no rain come out. He gave me a grab and said, "I'm here. Come back to my house tonight." I thought I should go to school with friends. That thinking put me in the wrong place. My mind blew up. Once I was sick, I couldn't go to school. Whenever the rain was coming over, I went outside and let no rain come out. I wanted to \repbox{sit in a row} and let no rain come out. I went to my mom's apartment and she yelled and shouted and said, "Come on if you don't want to." I looked at her and she shouted, shouted, put gunshots on my face. I went out of the apartment and fleeing. I drew into the apartment where I lived. I had to go outside leaving bullets and gunshots pinning on the door. I had to open my door, let no rain come out, then let no rain come out. I ran out of the apartment again and felt infuriated, fearless. It couldn't go right. I dived. I thought something was going to go wrong. I thought something was going wrong. I didn't know. Even now, I just thought I had to go to school, open my door, pick up a job. I didn't seem to find a place to go home. I felt disappointed and outraged. Despite all this, I felt locked, I felt locked, and I finally left home for the last moment. I haven't left since, because it makes me feel like I'm alive again. All I'm going to do is be away again. What hurts me is that I won't be home again. Advertisement Advertisement Share in this link: Tumblr Email Facebook Twitter Google Print Reddit Tumblr The Green Party is in talks through the general election season, and Tony Abbott's government has had to rescind the formal membership of the National Assembly on the eve of the general election.
\end{mdframed}

\begin{mdframed}[style=examplebox]
\textbf{ACE}\par\smallskip
\small My dad died when I was 5. We were enormously different friends, but we were all completely different. We lived in different families since I was about 7. To my dad, he was very unmarry. Then he got very excited. He was very excited when he realized it didn't matter that much. He loved to sit in the coffee table, listen to magazines and watch television. When he was two years younger, he started making drawings when he was ten. He fled to France from France to play video games or study history at Columbia University. He recently started making drawings for a small coffee shop in Britton, MI. He gave me the hope that he could find a job again. It ended up being my friend's last day when he was 25. I had thought about my dad for years, but he was very scared. When he started working too hard, he was scared. We went to work together for five years, at a coffee shop in East Hill Drive. My dad sat behind the table and I heard my dad yell. I really wanted to help him, but he wasn't on my side. He was fantastic. He had a bright face and a very bright eyeout. I wanted to help him win the job. Then he died at age 25. My dad took his entire life very seriously. It was one of the best memories I had ever had. I'd never met a guy like myself who would have seen my dad walking around a friend's table or watched my dad throwing an explosive with a finger. My dad was so lucky that I hadn't seen anything on public television or on social media since the day I was 11. It felt like my dad had a great family. I didn't know anyone. We had a completely different family. My dad and I met for five years. I would go to school for two or three years, I thought. Then we would hang out at home only once a week. Then I would spend the night drawing and making my own drawings. I took a graduate law class when I was five and studied law when I did. I went to an graduate law class as part of my career at a law school. My dad also died when I was 10. I remember a jack slurning on the ground by the wheel of a bus when my dad grew up. He and I were having so much fun together when we met a man found guilty of robbery in his apartment several blocks away. I saw a man cropping the wall of my dad's apartment as we tried to support each other. His voice made me even louder. He gave me a sore ovation to my dad, and he told me that he was about to tell his story about his dad. That story connected me to a story that my dad grew up with until I was 25. I and my dad were walking down the road near East Hillfield Drive on a recent weekend when a person was allegedly kidnaped and assaulted at a coffee shop. As I drove into the kitchen at the local coffee shop, someone kidnapped a man with his arm turned on and shouting, "Fuck free!" I called on phone and heard a young woman bragling on the abtainee's face. I and I knew my dad had assaulted. We drove down the road into an unlucky shop a few few blocks away. Shortly after a bomb pounded through the door, I had to open it to force a fire. It was fired and just started getting burned. I went back to the local shop and felt saddened and desperated. I couldn't believe what happened that night. I couldn't believe the alleged abtainee was the alleged attacker. I couldn't believe what happened next. I couldn't believe what happened next. David Zimmerman / Getty Images The British photographer and author David Zimmerman - the permanent owner of Getty Images has his latest narrar, 'My Tales About the Cold Wars,', a six-month-long series about a US military veteran that tells the story of a dissatisfied veteran who obsesss his own life while confronting with a fearless past that has transformed into a fearless past. The first issue is published in the curse of a new book, 'Waking the Top' - an eight-part series about the Cold Wars. The Obama administration and the FBI have helped investigate the assassing of Eric Lavarte on the night of August 9th.
\end{mdframed}

   % paired qualitative examples
\section{Related Work}
\label{app:related}

\paragraph{Self-conditioning in continuous diffusion language models.}
Self-conditioning reuses previous clean predictions for iterative
refinement~\citep{chen2022analog,strudel2022self}. Continuous text diffusion spans embedding-space
models such as Diffusion-LM~\citep{li2022diffusion}, Plaid~\citep{gulrajani2023likelihood},
LangFlow~\citep{chen2026langflow}, and ELF~\citep{hu2026elf}, as well as latent-space approaches
such as CoLA-DLM~\citep{guo2026continuous} and AURORA-LM~\citep{liang2026aurora}. Among these,
Plaid, LangFlow, ELF, and AURORA-LM use explicit self-conditioning, feeding previous clean
predictions back during denoising.
AURORA-LM generates latent blocks from left to right while denoising positions within each block.
Other continuous formulations include statistical-manifold diffusion~\citep{jo2025continuous}
and continuous diffusion with an autoregressive decoder~\citep{shen2026codar}. Recent analyses examine
how self-conditioning shapes denoising: Fixed-point Flows~\citep{yoo2026self} interpret it as
fixed-point iteration and exploit it for few-step distillation, while Untied Self-Conditioning
~\citep{li2026improving} addresses a training--inference mismatch by dampening redundant components
and correcting solver updates without retraining. We instead study how the self-conditioning loop
relates to repetition during unconditional generation.

\paragraph{Repetition in language generation.}
Autoregressive repetition has been linked to likelihood-maximizing decoding~\citep{holtzmancurious}
and self-reinforcement from repeated context~\citep{xu2022learning}. Remedies include unlikelihood
training, which penalizes repetitive outputs~\citep{welleckneural}, and contrastive search, which
balances model confidence with a representation-based degeneration penalty~\citep{su2022contrastive}.
Evaluation studies also show that low Gen-PPL need not indicate high generation
quality~\citep{wang2022perplexity,pynadath2026generative,franca2026hacking}.
Our study instead examines how feeding clean-embedding predictions back across denoising steps
produces repetition, in contrast to self-reinforcement through an autoregressive prefix.

\paragraph{Inference-time steering.}
Contrastive Activation Addition estimates steering vectors from average differences in residual-stream
activations between positive and negative behavioral examples~\citep{rimsky2024steering}.
Activation steering has also been extended to masked DLMs, where residual-stream directions
are applied during denoising to control model behavior~\citep{shnaidman2025activation}.
Motivated by our repetition-attractor analysis, ACE uses a related contrastive estimator but
intervenes on self-conditioning predictions rather than residual-stream activations.
    % related work

\end{document}